\documentclass[letterpaper,twocolumn,10pt]{article}
\usepackage{usenix2019_v3}

\usepackage{graphicx}
\usepackage{amsmath,amssymb,amsfonts}
\usepackage[mathcal]{euscript}
\usepackage{amsthm}
\usepackage{booktabs}
\usepackage{array}
\usepackage{textcomp}
\usepackage{tikz}
\usetikzlibrary{arrows.meta,positioning,fit,calc,shapes.geometric}
\usepackage{algorithm}
\usepackage{algorithmicx}
\usepackage{algpseudocode}
\usepackage[title]{appendix}

\newcommand{\EBFT}{\textnormal{\textsc{EBFT}}}
\newcommand{\BFT}{\textnormal{\textsc{BFT}}}
\newcommand{\SMR}{\textnormal{\textsc{SMR}}}
\newcommand{\SQA}{\textnormal{\textsc{SQA}}}
\newcommand{\PDDS}{\textnormal{\textsc{PDDS}}}

\newcommand{\Valid}{\mathsf{valid}}

\newcommand{\EpistemicBudget}[1][\delta]{e_{#1}}
\newcommand{\UnusableBudget}[1][\epsilon]{u_{#1}}
\newcommand{\SemanticSafetyEvent}[1][\delta]{\mathcal{E}_{#1}}
\newcommand{\LivenessEvent}[1][\epsilon]{\mathcal{U}_{#1}}
\newcommand{\EstimatedEpistemicBudget}[1][\delta]{\widehat{e}_{#1}}
\newcommand{\EstimatedUnusableBudget}[1][\epsilon]{\widehat{u}_{#1}}
\newcommand{\EstimatedEpistemicUpperBudget}[1][\delta]{\widehat{e}^{+}_{#1}}
\newcommand{\EstimatedUnusableUpperBudget}[1][\epsilon]{\widehat{u}^{+}_{#1}}
\newcommand{\CommonModeEvent}{\mathcal{C}}

\newcommand{\Prob}{\mathbb{P}}

\newcommand{\ValidatorSet}{\mathcal{R}}
\newcommand{\ByzantineSet}{\mathcal{B}}
\newcommand{\HonestSet}{\mathcal{H}}
\newcommand{\ProposalSpace}{\mathcal{X}}
\newcommand{\StateSpace}{\mathcal{S}}
\newcommand{\SemanticValidity}{\mathcal{V}}
\newcommand{\Canonicalize}[2]{\pi_{#2}(#1)}
\newcommand{\InvalidClassSet}[1]{\mathcal{Z}_{#1}^{-}}

\newcommand{\Commit}{\mathsf{commit}}
\newcommand{\Reject}{\mathsf{reject}}
\newcommand{\Abstain}{\mathsf{abstain}}
\newcommand{\Escalate}{\mathsf{escalate}}

\newcommand{\NumValidators}{N}
\newcommand{\ByzantineBound}{f}
\newcommand{\CoherentFaultBound}{e_c}
\newcommand{\DispersiveFaultBound}{e_i}

\newcommand{\StateContext}{s}

\newcommand{\botrule}{\bottomrule}

\hypersetup{
  hypertexnames=false,
  pdftitle={The Honest Quorum Problem: Epistemic Byzantine Fault Tolerance for Agentic Infrastructure},
  pdfauthor={Jun He and Deying Yu},
  pdfsubject={Epistemic fault tolerance and semantic quorum safety in agentic infrastructure},
  pdfkeywords={epistemic Byzantine fault tolerance, agentic infrastructure, Byzantine fault tolerance, semantic validity, correlated failures, distributed consensus}
}

\makeatletter
\let\standardtitle\title
\renewcommand{\title}[2][]{%
  \standardtitle{#2}%
  \def\shorttitle{#1}%
  \ifx\shorttitle\@empty\else\markboth{#1}{#1}\fi
}
\makeatother

\newcolumntype{L}[1]{>{\raggedright\arraybackslash}p{#1}}
\newcolumntype{C}[1]{>{\centering\arraybackslash}p{#1}}
\numberwithin{equation}{section}

\theoremstyle{plain}
\newtheorem{theorem}{Theorem}[section]
\newtheorem{lemma}[theorem]{Lemma}
\newtheorem{proposition}[theorem]{Proposition}
\newtheorem{corollary}[theorem]{Corollary}
\theoremstyle{definition}

\newtheorem{assumption}[theorem]{Assumption}

\newtheorem{definition}[theorem]{Definition}
\theoremstyle{remark}

\begin{document}

\title[The Honest Quorum Problem]{
The Honest Quorum Problem:
Epistemic Byzantine Fault Tolerance for Agentic Infrastructure
}

\author{
  {\rm Jun He}\\
  OpenKedge.io\\
  \texttt{junhe@openkedge.io}
  \and
  {\rm Deying Yu}\\
  OpenKedge.io\\
  \texttt{deying@openkedge.io}
}

\maketitle

\begin{abstract}
State machine replication and Byzantine fault-tolerant consensus guarantee
agreement despite a bounded number of arbitrary faulty participants; Byzantine
participants may coordinate or collude. These guarantees rely on participants
outside that set correctly implementing the protocol's transition or validation
semantics. Agentic validators expose a weaker partition: an authenticated,
responsive, correctly signed, non-equivocating reasoning participant that is
protocol-compliant with the voting protocol may nevertheless endorse a
semantically invalid state transition.

We call the resulting failure mode an epistemic fault and the collective
phenomenon the Honest Quorum Problem; here, honest means protocol-compliant,
not semantically correct. Such a quorum can satisfy ordinary protocol checks
while forming a well-formed certificate for an invalid transition.
We show that agreement alone does not establish semantic certificate validity
or execution safety. Agentic validators may share model weights or lineage, training
distributions, prompts, retrieval sources, toolchains, evidence, reasoning
scaffolds, and provider infrastructure, yielding correlated epistemic faults.

We define Epistemic Byzantine Fault Tolerance (\EBFT{}), a fault-tolerance
model for agentic infrastructure and other post-deterministic distributed systems.
\EBFT{} augments the conventional Byzantine fault bound with separate
confidence-indexed quantities: $e_\delta$ bounds coherent invalid endorsement
outside the Byzantine set, and $u_\epsilon$ bounds unusable validator support
that degrades liveness. These quantities characterize
semantic-safety risk and liveness degradation separately. The paper derives
certificate-threshold conditions for semantic certificate validity, consensus
agreement, liveness, and feasible quorum-threshold selection, and outlines a
calibration methodology for estimating the budgets. Adding nominally distinct
agents improves fault tolerance only when it measurably reduces the upper-tail
concentration of invalid endorsement or unusable support.

\end{abstract}


\section{Introduction}\label{sec:introduction}

An operator proposes an infrastructure state mutation: expand a deployment
service account from read-only inspection to write authority over a production
namespace. Several reasoning validators inspect the same canonical request,
state snapshot, policy context, and evidence package. Each validator
authenticates correctly, receives the same request digest, follows the voting
protocol, signs the expected message, responds before the timeout, and does
not equivocate. A quorum approves the transition, but the transition violates
an application invariant: it crosses the intended control-plane isolation
boundary. The protocol succeeds in forming agreement, yet the system commits a
semantically invalid action.

Classical state machine replication (\SMR{}) and Byzantine fault-tolerant
(\BFT{}) consensus already permit arbitrary faulty replicas to coordinate,
collude, equivocate, and choose adversarial messages
\cite{pease1980reaching,lamport1982byzantine,schneider1990state}. The
classical partition is different: replicas outside the declared Byzantine set
are assumed to follow the protocol and correctly implement the relevant
transition or validation semantics. The Honest Quorum Problem weakens this
partition. A participant may remain inside the identity, communication, and
protocol envelope while failing at the semantic layer.

Probabilistic reasoning networks are replicated or distributed systems in
which proposals, plans, approvals, or validations are generated by stochastic
reasoning processes. Agentic infrastructure is a useful stress case because
agent decisions can mutate operational state. We use
post-deterministic distributed system (\PDDS{}) \cite{pdds2026manifesto} to describe a system in which protocol
coordination can remain deterministic while semantically relevant outputs are
produced by stochastic, learned, adaptive, or generative computations. The
point is not to discard deterministic coordination; it is to model what
happens when deterministic coordination surrounds nondeterministic semantic
judgment.

An epistemic fault is a semantically incorrect judgment produced by a
participant that remains authenticated, responsive, non-equivocating, and
compliant with the base protocol. Multiple protocol-compliant validators may
fail coherently because they share model lineage, training distributions,
prompt patterns, retrieval sources, toolchains, evidence, reasoning
scaffolds, or provider infrastructure. Nominal identity diversity therefore
does not imply epistemic independence. Here, ``honest'' refers only to
identity-, communication-, and protocol-layer compliance; it does not imply
semantic correctness.

\EBFT{} denotes Epistemic Byzantine Fault Tolerance, the model developed
here. \EBFT{} does not replace classical \BFT{}. It augments the
conventional bound on arbitrary Byzantine identities with explicit bounds on
coherent invalid endorsement by protocol-compliant validators and on unusable
support caused by rejection, abstention, timeout, divergent output,
canonicalization failure, or unresolved uncertainty.

The analysis separates three properties:
\[
\begin{aligned}
\text{Protocol agreement} &\neq \text{Semantic validity},\\
\text{Semantic validity} &\neq \text{Execution safety}.
\end{aligned}
\]
Protocol agreement concerns whether protocol-compliant replicas decide the same
certificate or log entry. Semantic certificate validity concerns whether a
certified candidate satisfies the application's semantic-validity predicate in
the given state and evidence context. Execution safety concerns whether
applying the certified transition preserves the operational invariant of
interest. Semantic validity implies execution safety only when the predicate
fully captures that invariant.

One threshold cannot be obtained by merely reweighting the classical
$3f+1$ expression. \EBFT{} instead derives separate conditions for preventing
an invalid candidate from collecting a certificate, preventing conflicting
certificates, and ensuring that a valid candidate can still collect enough
usable support.

\noindent\textbf{The Honest Quorum Problem.} A distributed admission protocol
exhibits the Honest Quorum Problem when a certificate can be formed by
authenticated, responsive, non-equivocating, and protocol-compliant validators
for a transition that violates the application's semantic-validity predicate.

\EBFT{} is distinct from Practical Byzantine Fault Tolerance, commonly
abbreviated PBFT, and from probabilistic \BFT{} protocols that use random
sampling or probabilistic quorums to reduce communication. In \EBFT{},
probability characterizes the concentration of semantic failures among
protocol-compliant reasoning participants.

The paper makes four contributions. First, it formalizes the Honest Quorum
Problem and the epistemic fault model: protocol-compliant stochastic validators
can collectively endorse a semantically invalid transition, a failure mode
distinct from arbitrary Byzantine, crash, omission, benign, deceitful, or
equivocating faults. Second, it defines confidence-indexed epistemic budgets:
$e_\delta$ for coherent invalid endorsement and $u_\epsilon$ for unusable
support, with both quantities excluding the conventional Byzantine set unless
explicitly stated otherwise. Third, it derives \EBFT{} certificate-threshold
conditions for semantic certificate validity, agreement, liveness, and
feasible threshold selection without claiming tightness beyond the proved
intersection and threshold statements. Fourth, it formulates diversity-aware
assurance: model, prompt, retrieval, provider, and tool diversity are candidate
mechanisms for reducing correlation, not proofs of reduced tail risk.

This work fits into the OpenKedge research program as positioning, not as a
product description. Post-Deterministic Distributed Systems names the class of
systems in which deterministic coordination surrounds stochastic semantic
judgment. The Honest Quorum Problem identifies a failure mode in that class.
\EBFT{} supplies the corresponding fault model and threshold theory. Semantic
Quorum Assurance can instantiate EBFT-style admission by
constructing semantic certificates, while the Sovereign Assurance Boundary is
the runtime enforcement boundary that admits or rejects execution based on
those certificates \cite{he2026openkedge,he2026sqa,he2026sab}.

Section~\ref{sec:background} places \EBFT{} relative to classical consensus,
hybrid fault models, common-mode failure, and probabilistic consensus.
Section~\ref{sec:system-model} defines the semantic-certificate model.
Section~\ref{sec:pbf-model} formalizes the Honest Quorum Problem and
epistemic fault budgets. Sections~\ref{sec:impossibility}
and~\ref{sec:quorum-theory} give the impossibility and threshold results.
Section~\ref{sec:diversity-protocol} formulates diversity-aware semantic
certification, Section~\ref{sec:evaluation} gives the evaluation methodology,
and Sections~\ref{sec:related-work}--\ref{sec:conclusion} discuss related
work, limitations, and conclusions.

\vspace{-0.8\baselineskip}
\section{Background and Motivation}\label{sec:background}

\subsection{Byzantine Agreement, SMR, and Quorum Intersection}

Byzantine agreement asks correct processes to decide consistently despite a
bounded set of processes that may behave arbitrarily
\cite{pease1980reaching,lamport1982byzantine}. Partially synchronous
consensus and practical \BFT{} protocols show how agreement can be engineered
with authenticated messages and quorum certificates
\cite{dwork1988consensus,castro1999practical}. Byzantine quorum systems
capture the set-intersection argument behind these protocols: two conflicting
certificates cannot both be formed unless too much faulty weight participates
\cite{malkhi1998byzantine}.

\SMR{} uses consensus to order commands and then relies on deterministic
execution to replicate state \cite{schneider1990state}. Replicas outside the
faulty set are assumed to apply an equivalent transition function to the same
agreed log entry, so equal logs induce equivalent states. This assumption is
what lets the protocol separate ordering from execution. \EBFT{} keeps the
ordering theory intact but asks what changes when the decision that a
transition is semantically admissible is supplied by stochastic validators
rather than by a deterministic transition checker.

\subsection{Validated Consensus and External Validity}

Consensus specifications usually include a validity property in addition to
agreement and termination. Some validity notions are weak, requiring only that
decisions not be vacuous under favorable inputs; others are external or
validated, requiring that a decided value satisfy an application predicate.
A protocol may therefore require a deterministic predicate
$\Valid(v)$ before a value is eligible for certification. When every
protocol-compliant replica can evaluate the same predicate and obtain the same
answer, conventional external validity can be folded into the consensus
interface.

Civit et al.'s ``On the Validity of Consensus'' studies validity properties as
first-class objects in Byzantine consensus, including which validity
properties are solvable and the message-complexity cost of non-trivial
validity \cite{civit2023validity}. \EBFT{} does not invent validity as a
consensus property. Its concern is narrower and more operational: semantic
validity may be too expensive, contextual, or underspecified to evaluate by a
deterministic predicate inside the protocol, so a committee approximates it
using stochastic reasoning. The question then becomes whether those
approximations can concentrate on the same invalid semantic class while still
producing a well-formed certificate.

\subsection{Hybrid Fault Models}

Hybrid fault models refine the fully Byzantine abstraction by distinguishing
failure classes. Basilic, for example, introduces the
Byzantine-deceitful-benign model with $t$ Byzantine processes, $d$ deceitful
processes, and $q$ benign processes. For consensus, Basilic gives the tight
resilience condition $n>3t+d+2q$; for eventual consensus in the blockchain
setting it gives the weaker condition $n>2t+d+q$
\cite{ranchal2022basilic}. The distinction is useful because some faults
primarily threaten safety, while others primarily threaten termination.

\EBFT{} is related in spirit but not reducible to this taxonomy. A stochastic
semantic validator may be protocol-compliant, authenticated, responsive, and
non-equivocating while endorsing a semantically invalid transition. Such a
fault does not automatically map to a deceitful fault, because the validator
need not violate the protocol relation observed by other replicas; it does
not automatically map to a benign fault, because coherent invalid endorsement
can threaten semantic certificate validity; and treating it as fully Byzantine may be a
sound conservative reduction but loses the distinction between coherent
safety risk and dispersive liveness loss.

Table~\ref{tab:fault-model-comparison} summarizes these distinctions; in the
table, Y means yes, N means no, and C means conditional on the fault instance
or system assumptions.

\begin{table*}[tbp]
\caption{Fault capabilities relevant to protocol agreement, semantic certificate validity,
and liveness. Authentication means messages are attributable to the signing
identity; a Byzantine validator may still possess valid credentials and
signatures.}
\label{tab:fault-model-comparison}
\scriptsize
\setlength{\tabcolsep}{2pt}
\begin{tabular}{@{}L{0.14\textwidth}C{0.055\textwidth}C{0.07\textwidth}C{0.055\textwidth}C{0.07\textwidth}C{0.065\textwidth}C{0.075\textwidth}C{0.065\textwidth}L{0.25\textwidth}@{}}
\toprule
Fault model & Auth. & Protocol compliant & May equivocate & May falsely endorse & Threatens agreement & Threatens semantic certificate validity & Threatens liveness & Relevant correlation assumption \\
\midrule
Byzantine faults & Y & N & Y & Y & Y & Y & Y & Byzantine identities may coordinate or collude; no independence is assumed. \\
Crash or omission faults & Y & C & N & N & N & N & Y & Timing, availability, or omission failures may share infrastructure causes. \\
Benign faults & Y & C & N & N & N & N & C & Benign behavior may reduce usable support but does not create false endorsement. \\
Epistemic faults & Y & Y & N & Y & N & C & C & Semantic errors may arise from model, prompt, retrieval, evidence, tool, or provider overlap. \\
Coherent epistemic faults & Y & Y & N & Y & N & Y & C & Shared semantic failure domains concentrate support on one invalid result. \\
Dispersive epistemic faults & Y & Y & N & C & N & C & Y & Errors, abstentions, rejections, timeouts, or divergent classes fail to concentrate into one certificate. \\
\botrule
\end{tabular}
\vspace{2pt}

\emph{Legend:} Y = yes; N = no; C = conditional on the fault instance or system assumptions.
\end{table*}

\subsection{Common-Mode Failure and Design Diversity}

Common-mode failure is the reason diversity must be measured rather than
assumed. N-version programming proposed independently developed software
versions as a way to reduce coincident design faults
\cite{avizienis1985nversion}. Knight and Leveson's experimental evaluation
challenged the independence assumption by showing that separately produced
program versions can fail together \cite{knight1986experimental}. For \EBFT{},
the analogous quantity is not the number of nominally different validators,
providers, prompts, or tools. It is the upper-tail concentration of validator
weight assigned to one invalid semantic class.

Recent work on model ensembles reinforces the same point for stochastic
reasoning systems. Denisov-Blanch et al. find that polling-style aggregation
for truthfulness can fail to improve over a single sample and can amplify
shared misconceptions when model errors are correlated
\cite{denisovblanch2026consensus}. Turkmen et al. model correlated errors in
LLM ensembles and report saturation effects that limit majority-vote gains
relative to independence-based expectations \cite{turkmen2026ensemble}.
Code-generation studies point in the same direction: the Code Triangle study
reports clustered errors and limited diversity in LLM code implementations,
with model mixtures improving robustness but not eliminating the need to
measure failure concentration \cite{zhang2025codetriangle}. These results do
not establish a universal law about all validators; they motivate why \EBFT{}
tracks concentration rather than relying on diversity labels.

\subsection{Probabilistic Fault Models}

Probability already appears in consensus, but usually at the protocol layer.
Randomized Byzantine agreement uses random choices to obtain progress or
probabilistic guarantees under adversarial scheduling
\cite{benor1983another}. ProBFT takes a different but still protocol-level
route: it uses probabilistic Byzantine quorums and verifiable random functions
to reduce communication while guaranteeing safety and liveness with high
probability in a partially synchronous permissioned setting
\cite{avelas2024probft}. In such systems, probability concerns quorum
sampling, leader or committee selection, message complexity, or the chance
that the protocol's safety and liveness conditions hold.

\EBFT{} uses probability for a different object. The protocol messages may be
deterministic, authenticated, and correctly sampled; the uncertainty lies in
the semantic endorsement produced by validators whose reasoning procedures
may share data, tools, prompts, model lineages, or latent assumptions.
Probabilistic quorum sampling asks whether enough protocol-compliant
participants are sampled or intersect. Probabilistic semantic analysis asks
how much protocol-compliant support may endorse the same invalid transition.

\subsection{Stochastic Semantic Validation}

The systems considered here place a semantic-certificate layer between a
candidate transition and execution. Validators do not merely check syntax or a
locally decidable predicate; they may inspect evidence, retrieve context,
summarize policies, call tools, or infer whether a proposed transition
satisfies an application-level invariant. This is the setting in which
external validity becomes stochastic semantic validation.

The relevant failure mode is therefore certificate-level concentration. A
single erroneous validator can be harmless if other validators reject or
abstain. Many erroneous validators can be a liveness problem if their errors
scatter across incompatible classes. The safety-relevant event is more
specific: enough protocol-compliant validators endorse the same semantically
invalid transition to make it certificate-eligible. The gap is therefore
precise: existing work does not give a certificate-level model for the
upper-tail weight of protocol-compliant validators endorsing the same
semantically invalid transition.

\section{System and Semantic Model}\label{sec:system-model}

\subsection{Participants, Network, and Corruption}

The base model is a semantic-certificate protocol for admitting one proposed
transition. Let
$\ValidatorSet=\{1,\ldots,\NumValidators\}$ be the replica set. The Byzantine
set is $\ByzantineSet\subseteq \ValidatorSet$, with
$|\ByzantineSet|\le \ByzantineBound$, and the protocol-compliant validator set
is
\[
  \HonestSet=\ValidatorSet\setminus\ByzantineSet .
\]
Thus $\ByzantineSet\cap\HonestSet=\emptyset$ and
$\ByzantineSet\cup\HonestSet=\ValidatorSet$. Validators have equal weight
in the base model: each identity contributes one unit of support to a
certificate, and a quorum threshold is an integer $q$.

Channels are authenticated and the network is partially synchronous in the
standard sense used by many \BFT{} protocols
\cite{dwork1988consensus,castro1999practical}. The first quorum theorem uses
static Byzantine corruption: $\ByzantineSet$ is fixed for the execution under
analysis.
Byzantine validators may coordinate, equivocate, withhold messages, or send
arbitrary protocol messages subject only to authentication. Validators in
$\HonestSet$
are protocol-compliant and non-equivocating in the base model, but they may be
epistemically faulty: their semantic reasoning may incorrectly endorse an
invalid transition or incorrectly reject a valid one.

\begin{table*}[tbp]
\caption{Notation used for the \EBFT{} model.}
\label{tab:notation}
\footnotesize
\setlength{\tabcolsep}{3pt}
\renewcommand{\arraystretch}{0.92}
\begin{tabular}{@{}L{0.16\textwidth}L{0.80\textwidth}@{}}
\toprule
Symbol & Meaning \\
\midrule
$\ValidatorSet$ & Replica or validator set. \\
$\ByzantineSet$ & Byzantine validator set, with $|\ByzantineSet|\le \ByzantineBound$. \\
$\HonestSet$ & Protocol-compliant validator set, $\HonestSet=\ValidatorSet\setminus\ByzantineSet$. \\
$\NumValidators$ & Number of validators, or total validator weight in a weighted generalization. \\
$\ByzantineBound$ & Bound on arbitrary Byzantine validators. \\
$q$ & Signature threshold for a semantic certificate. \\
$\EpistemicBudget$ & Semantic-safety budget value: a confidence-indexed bound on coherent false endorsement outside the Byzantine set. \\
$\UnusableBudget$ & Liveness budget value: a confidence-indexed bound on protocol-compliant support unusable for liveness. \\
$\SemanticSafetyEvent$ & Event on which false-endorsement weight is bounded by $\EpistemicBudget$. \\
$\LivenessEvent$ & Event on which unusable support is bounded by $\UnusableBudget$. \\
$\EstimatedEpistemicBudget$ & Empirical estimate of $\EpistemicBudget$ from calibration data. \\
$\EstimatedUnusableBudget$ & Empirical estimate of $\UnusableBudget$ from calibration data. \\
$\EstimatedEpistemicUpperBudget$ & Upper confidence endpoint for $\EpistemicBudget$ used in threshold selection. \\
$\EstimatedUnusableUpperBudget$ & Upper confidence endpoint for $\UnusableBudget$ used in threshold selection. \\
$s$ & System state or semantic context. \\
$x$ & Candidate transition submitted for semantic certification. \\
$\SemanticValidity(x,s)$ & Semantic-validity predicate for candidate $x$ in context $s$. \\
$\pi_s(x)$ & Semantic canonicalization function for context $s$. \\
$F_S(x,s)$ & False-endorsement weight among protocol-compliant validators in committee $S$. \\
$U_S(x,s)$ & Unusable protocol-compliant support for a valid candidate. \\
$C_S(s)$ & Class-level coherent invalid support in the proposal-generation extension. \\
$\delta,\epsilon$ & Permitted upper-tail probabilities for semantic safety and liveness budgets. \\
\botrule
\end{tabular}
\end{table*}

\subsection{Admission Instances}

An admission instance begins when a proposer submits a candidate transition
$x\in\ProposalSpace$ for a state or context $s\in\StateSpace$. Every
protocol-compliant validator receives the same canonical evidence package
\[
  h_s = h(s,\mathsf{evidence},\mathsf{policy}),
\]
which binds the state/context, the evidence supplied for semantic assessment,
and the policy under which the assessment is made. The instance identity
contains the protocol view or round, a digest of $x$, the state/context
identifier, and the evidence-package identity $h_s$.

Evidence-package identity is a protocol condition, not a semantic claim. A
validator in $\HonestSet$ treats two requests as the same admission instance
only when their view, candidate digest, state/context identifier, and $h_s$
agree. If these fields are inconsistent, the validator rejects or abstains
according to the protocol. Context consistency therefore prevents
protocol-compliant validators from combining judgments about different
evidence packages into one certificate, while still allowing Byzantine
validators to equivocate.

\subsection{Semantic Validity and Classes}

Let $\SemanticValidity:\ProposalSpace\times\StateSpace\to\{0,1\}$ be the
application semantic-validity predicate. The predicate represents the
application-level condition that a transition must satisfy before execution. It
need not be cheaply decidable by the ordering protocol.
$\SemanticValidity(x,s)$ is a specification-level semantic truth predicate used
for analysis; a deployment may only approximate, certify, or estimate it
through deterministic verifiers, executable tests, model checking,
adjudication, or calibrated labels.

\begin{definition}[Semantic validity]\label{def:semantic-validity}
A candidate transition $x$ is semantically valid in context $s$ if
$\SemanticValidity(x,s)=1$. It is semantically invalid if
$\SemanticValidity(x,s)=0$.
\end{definition}

The model does not assume uniqueness of valid candidates. For the same
context $s$, several distinct transitions may satisfy
$\SemanticValidity(x,s)=1$. The certificate protocol admits or rejects the
submitted candidate $x$; it does not by itself prove that $x$ is the only valid
transition available.

\begin{definition}[Semantic equivalence]\label{def:semantic-equivalence}
For a fixed context $s$, two candidates $x$ and $y$ are semantically
equivalent if they induce the same application-relevant outcome. We write
$\Canonicalize{x}{s}=\Canonicalize{y}{s}$ when a deterministic
canonicalization function maps them to the same semantic class.
\end{definition}

\begin{assumption}[Deterministic canonicalization]\label{assumption:canonicalization}
For a fixed context and evidence package, $\pi_s$ is deterministic and
equivalent application-relevant outcomes map to the same canonical
representative. The assumption is application-specific and must be justified by
the certificate schema used by a deployment.
\end{assumption}

\subsection{Signed Judgments and Certificates}

A validator returns a signed structured judgment
\[
  J_i =
  \left\langle
    i,\mathsf{view},\mathsf{digest}(x),s,h_s,a_i,c_i,d_i,p_i
  \right\rangle_{\sigma_i},
\]
where $a_i\in\{\mathsf{endorse},\Reject,\Abstain\}$ is the validator action,
$c_i$ is the reported semantic class, $d_i$ is the evidence digest actually
used by the validator, $p_i$ is a provenance record for the semantic
assessment, and $\sigma_i$ is the validator's signature. For abstentions, the
semantic class may be $\bot$ if the validator cannot assign a class. For
endorse or reject judgments, the class must be computed for the submitted
candidate and the bound evidence package.

A semantic certificate for $(s,x,h_s)$ is a set $Q$ of signed judgments from
distinct validators such that $|Q|\ge q$, every judgment has action
$\mathsf{endorse}$, and all judgments agree on the view, candidate digest,
state/context identifier, evidence-package identity, and reported semantic
class. We write $Q_{\HonestSet}=Q\cap\HonestSet$ and
$Q_{\ByzantineSet}=Q\cap\ByzantineSet$; these subsets are disjoint and
$Q=Q_{\HonestSet}\cup Q_{\ByzantineSet}$. A certificate may contain both Byzantine and
protocol-compliant validators. It may also contain validators that share model
lineage, retrieval systems, tools, training data, or operational provenance;
such overlap is captured by the stochastic semantic-fault model, not by
identity overlap in $\ByzantineSet$ and $\HonestSet$.

\subsection{Fault, Randomness, and Adversary Model}

Only Byzantine validators may equivocate in the base model. A
protocol-compliant validator in $\HonestSet$ signs at most one judgment for a given
admission instance. Epistemic faults are therefore semantic, not protocol
faults: an epistemically faulty validator follows the message format,
signature rules, evidence binding, and non-equivocation rule, but may endorse
a candidate with $\SemanticValidity(x,s)=0$ or reject a candidate with
$\SemanticValidity(x,s)=1$.

Probabilities are taken over an explicit semantic-validation probability space. The
sample space may include a workload or task distribution over states,
candidates, evidence, and policies; validator-local randomness; retrieval and
tool randomness; model versions and configuration choices; and latent
common-mode variables that induce correlated semantic errors. These sources
need not be independent. When the analysis conditions on a fixed admission
instance $(s,x,h_s)$, probabilities describe the validator and environment
randomness for that instance. When the analysis averages over a workload
distribution, the workload distribution must be named and the resulting bound
is a workload-level statement rather than a per-context guarantee.

The adversary is worst-case with respect to Byzantine behavior, message
scheduling before the synchrony bound, and the contents of Byzantine
judgments. Unless a theorem states otherwise, the adversary may choose the
Byzantine set subject to $|\ByzantineSet|\le \ByzantineBound$ and may choose
messages for Byzantine validators, but probability statements about
$\HonestSet$ are conditional on
the fixed context and evidence package under analysis. If an adversary is also
allowed to select the task, context, or evidence package after observing model
or workload information, the corresponding stochastic bound must hold
uniformly over that selection or be explicitly stated as an average-case
workload bound.

\subsection{Properties}

The following properties are intentionally separated because \EBFT{} concerns
the gap between protocol agreement and semantic correctness.

\begin{definition}[Agreement]\label{def:agreement}
Agreement holds when two protocol-compliant replicas do not decide different
certificates for the same log position or admission instance.
\end{definition}

\begin{definition}[Semantic certificate validity]\label{def:semantic-certificate-validity}
Semantic certificate validity holds when every certificate accepted for
candidate $x$ in context $s$ satisfies $\SemanticValidity(x,s)=1$.
\end{definition}

\begin{definition}[Conventional external validity]\label{def:external-validity}
Conventional external validity holds when every decided certificate satisfies
the protocol's deterministic admissibility predicate: well-formed signed
judgments, distinct validator identities, threshold $q$, agreement on the
candidate digest, context, evidence-package identity, view, and required
action. This property is syntactic and certificate-level; it does not by
itself imply $\SemanticValidity(x,s)=1$.
\end{definition}

\begin{definition}[Liveness]\label{def:liveness}
Liveness holds when, after the network stabilizes, a semantically valid
candidate submitted with a consistent evidence package eventually obtains a
certificate whenever at least $q$ protocol-compliant validators return usable
endorsements under the fault and stochastic-assessment assumptions.
\end{definition}

\begin{definition}[Execution safety]\label{def:execution-safety}
Execution safety holds when applying the certified transition preserves the
application's safety invariant. Semantic certificate validity implies
execution safety only when the chosen predicate $\SemanticValidity$ is strong
enough to encode the relevant invariant-preservation condition.
\end{definition}

\begin{definition}[Consensus properties]\label{def:properties}
For reference, this paper uses agreement, semantic certificate validity,
conventional external validity, liveness, and execution safety in the senses
defined above.
\end{definition}

\subsection{Generative Proposal Extension}

Some semantic-validation systems ask validators to generate candidate
transitions rather than only assess a proposer-supplied transition. That is a
separate extension of the base admission model. In the extension, validator
$i$ may output a proposed transition $x_i$ or a proposed semantic class after
receiving the task, context, evidence, and policy. Coherent invalid support can
then be measured over the multiset of generated transitions or classes, as in
the notation used for semantic-fault concentration in Section~\ref{sec:pbf-model}.

The extension inherits the same sets $\ByzantineSet$ and $\HonestSet$, the same authenticated and
partially synchronous network, and the same non-equivocation requirement for
protocol-compliant validators. It additionally requires an explicit proposal
aggregation rule, canonicalization rule, and evidence-binding rule before it
can be used in a quorum theorem. Unless a result names this extension, the
paper's base admission results concern a fixed candidate $x$ submitted by a
proposer and judged by validators through signed endorse, reject, or abstain
messages.

\section{The Honest Quorum Problem and Epistemic Fault Model}\label{sec:pbf-model}

Section~\ref{sec:system-model} fixes an admission instance in which validators
judge a proposer-supplied candidate transition $x$ for context $s$. For a
committee $S\subseteq \ValidatorSet$, write
$\HonestSet_S=S\cap\HonestSet$ for the protocol-compliant validators in the
committee. In the equal-weight base model $w_i=1$ for all $i\in S$; the
notation below keeps weights only to make the weighted generalization
explicit.

For each protocol-compliant validator $i\in \HonestSet_S$, let $J_i(x,s)$ be
the random structured judgment returned for the admission instance. Its action
component is one of $\mathsf{endorse}$, $\Reject$, or $\Abstain$, and its
randomness is taken over the probability space described in
Section~\ref{sec:system-model}: validator-local randomness, retrieval and tool
randomness, model versions, and latent common-mode variables. Byzantine
validators are not sampled by this semantic error process; they remain
worst-case protocol adversaries.

\subsection{Honest Quorums and Epistemic Faults}

\begin{definition}[The Honest Quorum Problem]\label{def:honest-quorum-problem}
A semantic-certificate protocol exhibits the Honest Quorum Problem when there
exists an admission instance $(x,s,h_s)$ with $\SemanticValidity(x,s)=0$ and a
set $Q\subseteq \HonestSet$ of authenticated, responsive, non-equivocating, and
protocol-compliant validators such that $|Q|\ge q$ and every validator in $Q$
returns a signed endorsement for the same candidate, context, evidence-package
identity, and semantic class. The quorum is honest at the identity,
communication, and protocol layers, not necessarily at the semantic layer; the
failure is that protocol agreement does not establish semantic validity.
\end{definition}

\begin{definition}[Epistemic fault]\label{def:epistemic-fault}
An epistemic fault is a semantically incorrect judgment generated by a
validator in $\HonestSet$: the validator follows the protocol interface and
non-equivocation rule, but endorses a semantically invalid candidate, rejects a
semantically valid candidate, abstains where usable support is required, or
returns a divergent semantic class for the bound evidence package.
\end{definition}

\begin{definition}[Coherent and dispersive epistemic faults]\label{def:coherent-dispersive}
A coherent epistemic fault occurs when multiple protocol-compliant validators
converge on the same invalid semantic judgment or endorsement. A dispersive
epistemic fault occurs when erroneous, abstaining, rejecting, timed-out,
divergent, or uncertain judgments do not concentrate enough support on one
invalid result to form an unsafe certificate, but may reduce liveness.
\end{definition}

\subsection{False Endorsement and Semantic-Safety Budgets}

\begin{definition}[False-endorsement weight]\label{def:false-endorsement-weight}
For an invalid candidate $x$ in context $s$, the false-endorsement weight
outside the Byzantine set is
\begin{equation}
F_S(x,s)
=
\sum_{i\in \HonestSet_S} w_i\,\mathbf{1}^{\mathrm{FE}}_i(x,s).
\label{eq:false-endorsement-weight}
\end{equation}
Here $\mathbf{1}^{\mathrm{FE}}_i(x,s)=1$ exactly when
$J_i(x,s)=\mathsf{endorse}$ and $\SemanticValidity(x,s)=0$, and is zero
otherwise.
\end{definition}

The summation is over $\HonestSet_S$, not over all validators, so $F_S$ excludes
Byzantine support by definition. Any theorem that combines Byzantine behavior
with semantic failures must add the Byzantine budget separately and state the
conditioning event under which the stochastic budget holds.

\begin{definition}[Pointwise semantic-safety budget]\label{def:pointwise-semantic-safety-budget}
For an invalid admission instance $(x,s)$, the pointwise semantic-safety budget
at confidence $1-\delta$ is
\begin{equation}
\EpistemicBudget(S;x,s)
=
\inf\left\{
e:
\Prob\!\left[F_S(x,s)>e \mid x,s,h_s\right]\le \delta
\right\}.
\label{eq:pointwise-semantic-safety-budget}
\end{equation}
\end{definition}

\begin{definition}[Workload semantic-safety budget]\label{def:semantic-safety-budget}\label{def:epistemic-budget}
For a workload support or distribution $D$, the workload-level semantic-safety
budget is
\begin{equation}
\EpistemicBudget(S,D)
=
\inf\left\{
e:
\Prob\!\left[
  \sup_{\substack{(x,s)\in D\\ \SemanticValidity(x,s)=0}}
  F_S(x,s)>e
\right]\le \delta
\right\}.
\label{eq:epistemic-budget}
\end{equation}
\end{definition}

The probability in Equation~\ref{eq:epistemic-budget} must be read with its
sampling convention. A pointwise guarantee conditions on a fixed
$(x,s,h_s)$; an average-case guarantee additionally samples from the workload
distribution $D$; a uniform guarantee requires the bound to hold over every
admissible context in the stated support. A workload-average guarantee can be
invalidated by adversarial task selection: if an adversary chooses contexts
after observing validator families, prompts, retrieval sources, or model
versions, rare high-risk contexts may be selected more often than their
workload probability indicates.

\subsection{Unusable Support and Liveness Budgets}

Semantic safety is not liveness. A committee may have small
$\EpistemicBudget(S,D)$ because few protocol-compliant validators falsely
endorse invalid candidates, while still failing to certify valid candidates
because many protocol-compliant validators reject, abstain, time out, disagree
on the semantic class, or fail canonicalization.

\begin{definition}[Unusable support]\label{def:unusable-support}
For a valid candidate $x$ in context $s$, let $U_i(x,s)=1$ when validator
$i\in \HonestSet_S$ produces support that cannot be used in an endorsement certificate.
Unusable support includes false rejection, abstention, timeout, divergent
semantic output for the bound evidence package, and canonicalization failure.
The total unusable support is
\begin{equation}
U_S(x,s)
=
\sum_{i\in \HonestSet_S} w_i U_i(x,s),
\qquad
\SemanticValidity(x,s)=1.
\label{eq:unusable-support}
\end{equation}
\end{definition}

\begin{definition}[Liveness budget]\label{def:liveness-budget}
For a workload support or distribution $D$, the confidence-indexed liveness
budget is
\begin{equation}
\UnusableBudget(S,D)
=
\inf\left\{
u:
\Prob\!\left[
  \sup_{\substack{(x,s)\in D\\ \SemanticValidity(x,s)=1}}
  U_S(x,s)>u
\right]\le \epsilon
\right\}.
\label{eq:liveness-budget}
\end{equation}
\end{definition}

\subsection{Confidence Events and Calibration Scope}

Both $\EpistemicBudget(S,D)$ and $\UnusableBudget(S,D)$ depend on committee
composition, workload domain, model versions, prompts, retrieval sources,
toolchains, evidence sources, confidence parameters, and calibration time.
They are properties of the underlying stochastic validation process, not
constants of the protocol. Empirical calibration can only estimate them:
$\EstimatedEpistemicBudget(S,D)$ and $\EstimatedUnusableBudget(S,D)$ should be
reported with sample sizes, stratification choices, adjudication procedure, and
confidence intervals. The true budget values remain the quantities in
Equations~\ref{eq:epistemic-budget} and \ref{eq:liveness-budget}.

We formalize the confidence events $\SemanticSafetyEvent$ and $\LivenessEvent$ as events over the underlying probability space $(\Omega, \mathcal{F}, \mathbb{P})$, where $\Omega$ represents the sample space of joint validator judgment outcomes $J_i(x,s)$ for all $i \in \HonestSet_S$ over the domain $D$. Specifically, the events are defined as:
\begin{align}
\SemanticSafetyEvent &= \left\{ \omega \in \Omega \;\middle|\; \sup_{\substack{(x,s)\in D\\ \SemanticValidity(x,s)=0}} F_S(x,s)(\omega) \le \EpistemicBudget(S,D) \right\}, \label{eq:safety-event-formal} \\
\LivenessEvent &= \left\{ \omega \in \Omega \;\middle|\; \sup_{\substack{(x,s)\in D\\ \SemanticValidity(x,s)=1}} U_S(x,s)(\omega) \le \UnusableBudget(S,D) \right\}. \label{eq:liveness-event-formal}
\end{align}
By construction, the definitions in Equations~\ref{eq:epistemic-budget} and \ref{eq:liveness-budget} imply that $\Prob[\SemanticSafetyEvent] \ge 1-\delta$ and $\Prob[\LivenessEvent] \ge 1-\epsilon$.

Several statistical quantities must be kept separate. Marginal error is the
probability that one validator gives the wrong action for a fixed instance.
Pairwise correlation measures dependence between two validators' error events.
Semantic collision measures whether erroneous validators land on the same
invalid semantic class. Upper-tail concentration is the probability that the
aggregate false-endorsement weight exceeds a certificate-relevant threshold.
Common-mode latent events are hidden causes, such as shared training artifacts,
retrieval failures, prompt vulnerabilities, or provider outages, that can move
many validators together. Pairwise correlation does not determine the quorum
tail: two validation processes can have the same marginal errors and pairwise
correlations while assigning very different probability to a rare event in
which many validators endorse the same invalid transition.

\begin{definition}[Epistemic Byzantine Fault Tolerance]\label{def:ebft}
\EBFT{} composes worst-case Byzantine behavior over $\ByzantineSet$ with
confidence-indexed semantic budgets over $\HonestSet_S$. The safety budget
$\EpistemicBudget(S,D)$ bounds false endorsement by protocol-compliant
validators outside the Byzantine set, while $\UnusableBudget(S,D)$ bounds
support that is epistemically unusable for liveness. The two budgets are distinct;
$\EpistemicBudget$ alone cannot establish liveness.
\end{definition}

\subsection{Proposal-Generation Extension}

The base model above concerns validators judging a fixed candidate. In a
proposal-generation extension, validators may instead generate candidate
transitions before a certificate rule groups their outputs. Let $\mathcal{Z}_s$
be the semantic class space for context $s$, and let
$\InvalidClassSet{s}\subseteq\mathcal{Z}_s$ be the set of invalid semantic
classes. If validator $i$ generates transition $X_i(s)$, its class is
$Z_i(s)=\pi_s(X_i(s))\in\mathcal{Z}_s$.

Coherent invalid support in the extension is a class-level quantity:
\begin{equation}
C_S(s)
=
\max_{z\in\InvalidClassSet{s}}
\sum_{i\in \HonestSet_S}
w_i \mathbf{1}\!\left[Z_i(s)=z\right].
\label{eq:coherent-invalid-support}
\end{equation}
Here $z$ is a semantic class, not a proposal. This extension requires its own
aggregation and canonicalization rules before it can be used in a quorum
theorem; the base semantic-safety budget is defined by false endorsements of a
proposer-supplied candidate through $F_S(x,s)$.

\section{Impossibility and Lower Bounds}\label{sec:impossibility}

The lower bounds below formalize the separation between protocol-visible
agreement and semantic correctness. A protocol transcript contains the fields
available to the consensus protocol: validator identities, signatures,
views or rounds, candidate digests, evidence-package identities, and structured
endorse, reject, or abstain judgments. The semantic truth of
$\SemanticValidity(x,s)$ is not protocol-visible unless it is supplied by an
oracle, verifier, or independently justified evidence. A concentration bound
does not reveal semantic truth; it only bounds the risk that
protocol-compliant validators concentrate enough false endorsement to form an
unsafe certificate.

\begin{theorem}[Agreement does not imply semantic certificate validity]
\label{thm:agreement-not-validity}\label{prop:agreement-not-validity}
There exists a legal execution in the base admission model, with a static
Byzantine set, in which all messages are authenticated, no protocol-compliant
validator equivocates, the consensus protocol satisfies agreement and
termination, and an invalid transition obtains a certificate. This is an
existence separation; it is deterministic once the protocol-compliant semantic
judgments are fixed, and it is not a necessary condition or a tight threshold
result.
\end{theorem}

\begin{proof}
Consider an admission instance with candidate $x$ and context $s$ such that
$\SemanticValidity(x,s)=0$. Let the Byzantine set be empty, so every validator
is protocol-compliant. Choose a set $Q\subseteq \HonestSet$ of distinct validators with
$|Q|\ge q$. For each $i\in Q$, let the random semantic judgment in this
execution be a signed endorsement
\[
  J_i(x,s)=\mathsf{endorse}
\]
for the same view, candidate digest, context identifier, evidence-package
identity, and semantic class. Validators outside $Q$ may abstain. Every
message is signed by its sender, so authentication holds. Each validator sends
at most one judgment for the admission instance, so no protocol-compliant
validator equivocates.

By the certificate rule from Section~\ref{sec:system-model}, the signed
endorsements from $Q$ form a conventional externally valid certificate for
$(s,x,h_s)$.
Run any consensus protocol that satisfies agreement and termination and whose
external validity rule accepts such a well-formed certificate. Since the
network is in a legal timely execution, termination gives a decision. Since
the protocol satisfies agreement, all protocol-compliant replicas decide the same
certified transition. The decided transition is nevertheless semantically
invalid because $\SemanticValidity(x,s)=0$. Hence agreement and termination do
not imply semantic certificate validity.
\end{proof}

\begin{theorem}[No count-only semantic guarantee without concentration]
\label{thm:count-only-impossible}\label{prop:count-only-impossible}
Let $\Gamma$ be any nontrivial certificate rule whose decision depends only on
signed validator identities and protocol-visible judgments. If the distribution
of protocol-compliant semantic errors is unrestricted, then there exists a
validator-output distribution and an invalid admission instance for which an
invalid candidate satisfies $\Gamma$ with probability one. This is an
impossibility result for agreement-visible, count-only evidence under a static
adversary and an unrestricted semantic-error distribution; it is not a numeric
tightness claim.
\end{theorem}

\begin{proof}
Because $\Gamma$ is nontrivial, there is some protocol-visible transcript
$T^\star$ that $\Gamma$ accepts. Otherwise the rule never certifies any
candidate and cannot be a live certificate rule. Let $I(T^\star)$ be the set
of validator identities whose signed judgments in $T^\star$ are used by
$\Gamma$, and let the visible judgments in $T^\star$ be well-formed
endorsements for a candidate digest, context identifier, evidence-package
identity, and view.

Construct two specification-level interpretations compatible with the same
protocol-visible transcript $T^\star$. In execution $E^+$, the semantic state satisfies
$\SemanticValidity(x,s^+)=1$, and the endorsing validators in $I(T^\star)$
return the signed judgments shown in $T^\star$. In execution $E^-$, the same
candidate digest, context identifier, evidence-package identity, signatures,
validator identities, and visible judgments occur, but the semantic state
associated with the admission instance satisfies $\SemanticValidity(x,s^-)=0$.
This represents the case in which the protocol-visible evidence and judgments
are compatible with both interpretations because they do not include a sound
semantic validity proof.

The two executions are indistinguishable to $\Gamma$ because $\Gamma$ is a
function only of the protocol-visible transcript, and that transcript is
$T^\star$ in both executions. Therefore $\Gamma$ accepts in $E^-$ exactly as
it accepts in $E^+$.

It remains to realize $E^-$ by an unrestricted semantic-error distribution.
Choose a distribution concentrated on the event that every protocol-compliant
validator whose identity appears in $I(T^\star)$ returns the same signed
endorsement shown in $T^\star$ for the invalid instance, while all other
validators return any judgments consistent with the transcript. This
distribution is allowed because protocol-compliant semantic errors are
unrestricted. Under that distribution, the invalid candidate satisfies
$\Gamma$ with probability one. Thus no rule that depends only on counts,
identities, and visible judgments can guarantee semantic certificate validity
without an additional concentration or semantic-grounding assumption.
\end{proof}

\begin{theorem}[Common-mode replication floor]
\label{thm:common-mode-floor}\label{prop:clone-non-amplification}
Let $G\subseteq \HonestSet$ be a replicated model family, and let
$\CommonModeEvent$ be a latent common-mode event under which every validator in $G$
falsely endorses the same invalid candidate $x$ in context $s$. If a threshold
rule is satisfied by the endorsements of $G$, then the unsafe-certificate
probability is at least $\Prob[\CommonModeEvent]$. This is a probabilistic lower
bound for a fixed committee and static Byzantine adversary; it is not an upper
bound and is not claimed tight unless all other unsafe events are ruled out.
\end{theorem}

\begin{proof}
Let $A$ be the event that an invalid candidate obtains a certificate accepted
by the threshold rule. By assumption, on the latent event $\CommonModeEvent$ every
validator in $G$ endorses the same candidate $x$, and
$\SemanticValidity(x,s)=0$. The threshold rule is satisfied by the endorsements
of $G$, so whenever $\CommonModeEvent$ occurs the invalid candidate obtains an
accepted certificate. Therefore $\CommonModeEvent\subseteq A$. Monotonicity of
probability gives
\[
  \Prob[A]\ge \Prob[\CommonModeEvent].
\]
Adding more clones from the same replicated family may change the conditional
variance of endorsements on $\neg\CommonModeEvent$, but it does not reduce
$\Prob[\CommonModeEvent]$ or the implication $\CommonModeEvent\subseteq A$ as long as
the cloned group still satisfies the threshold rule. The common-mode event is
therefore a floor on unsafe-certificate probability.
\end{proof}

\begin{theorem}[Necessity of external grounding]
\label{thm:external-grounding-necessary}\label{prop:apparent-diversity}
No protocol whose acceptance rule uses only agreement signals can guarantee
semantic certificate validity for all executions unless it assumes at least
one additional semantic grounding mechanism: a deterministic validity oracle,
a trusted verifier, a bounded false-endorsement distribution, or independently
justified semantic evidence. This is a necessity-of-assumption result for
protocols whose inputs are only protocol-visible agreement signals; it does
not assert that any one grounding mechanism is uniquely necessary. The argument
is deterministic.
\end{theorem}

\begin{proof}
Assume, for contradiction, that a protocol guarantees semantic certificate
validity for all executions while using only agreement signals: signed
identities, rounds or views, quorum counts, candidate digests, and
protocol-visible judgments. Also assume that none of the listed semantic
grounding mechanisms is available. Then the protocol has no deterministic
oracle or trusted verifier for $\SemanticValidity(x,s)$, no probabilistic
upper bound on false endorsement by protocol-compliant validators, and no
independently justified semantic evidence that rules out indistinguishable
valid and invalid interpretations of the same visible transcript.

Because the protocol is useful as a certificate protocol, it accepts some
protocol-visible transcript $T^\star$ for a certified candidate. By the
indistinguishability construction in Theorem~\ref{thm:count-only-impossible},
there is an execution with the same agreement signals and visible judgments in
which the certified candidate is semantically invalid. The protocol must make
the same acceptance decision in both executions because its inputs are the
same. It therefore accepts an invalid certificate in the second execution,
contradicting the assumed semantic certificate validity guarantee.

Agreement signals alone are therefore insufficient. A semantic certificate
validity guarantee needs some additional grounding assumption, such as an
oracle, verifier, bounded false-endorsement distribution, or independently
justified semantic evidence.
\end{proof}

These lower bounds show that agreement-visible evidence is insufficient for
semantic certificate validity. A system must add at least one semantic
grounding mechanism: deterministic validity checking, trusted verification,
independently justified evidence, or a calibrated concentration bound over
false endorsement.
Section~\ref{sec:quorum-theory} studies the last option: agreement constrains
how replicas order and decide protocol-visible certificates, while an explicit
budget bounds the probabilistic concentration of invalid semantic endorsement.

\section{\texorpdfstring{\EBFT{}}{EBFT} Safety, Agreement, and Liveness Thresholds}\label{sec:quorum-theory}

The base protocol is a $q$-signature semantic-certificate protocol. For a log
position or admission instance, a certificate for candidate $x$ in context $s$
is a set $Q$ of signed endorse judgments from distinct validators, all bound
to the same view, candidate digest, evidence-package identity $h_s$, and
semantic class. The certificate rule accepts exactly when $|Q|\ge q$. The
base model is equal-weight, so $\NumValidators=|\ValidatorSet|$ and
$q$ is an integer threshold.

Two certificates conflict when they certify different candidate digests,
contexts, evidence-package identities, or semantic classes for the same log
position. Protocol-compliant validators do not equivocate, so a validator in
$\HonestSet$ signs at most one endorse judgment for a log position. Byzantine
validators may equivocate or withhold. The following table separates the
assumptions needed for semantic certificate validity, agreement, and liveness.

\begin{table*}[tbp]
\caption{Assumptions used by the certificate-threshold theorems.}
\label{tab:threshold-assumptions}
\scriptsize
\begin{tabular}{@{}L{0.18\textwidth}L{0.39\textwidth}L{0.18\textwidth}L{0.17\textwidth}@{}}
\toprule
Property & Assumptions & Threshold condition & Guarantee \\
\midrule
Semantic certificate validity &
At most $\ByzantineBound$ Byzantine validators; on event $\SemanticSafetyEvent$,
protocol-compliant false endorsement of any invalid candidate is at most
$\EpistemicBudget$. &
$q>\ByzantineBound+\EpistemicBudget$ &
On $\SemanticSafetyEvent$, no invalid candidate forms a $q$-certificate. \\
Agreement &
Authenticated signatures; only Byzantine validators may equivocate; conflicting
certificates are for the same log position. &
$2q-\NumValidators>\ByzantineBound$ &
No two conflicting $q$-certificates both form. \\
Liveness &
Partial synchrony after GST; valid candidate and consistent evidence package;
Byzantine validators may withhold; on event $\LivenessEvent$, at most
$\UnusableBudget$ protocol-compliant support is unusable. &
$q\le \NumValidators-\ByzantineBound-\UnusableBudget$ &
A valid candidate can obtain a certificate on $\LivenessEvent$. \\
Feasibility &
Both stochastic events hold: $\SemanticSafetyEvent\cap\LivenessEvent$. &
$\max\{\ByzantineBound+\EpistemicBudget,(\NumValidators+\ByzantineBound)/2\}<q
\le \NumValidators-\ByzantineBound-\UnusableBudget$ &
Semantic certificate validity and liveness hold on their events; agreement holds deterministically. \\
\botrule
\end{tabular}
\end{table*}

\subsection{Semantic Certificate Validity, Agreement, and Liveness}

\begin{theorem}[Semantic certificate validity]\label{thm:semantic-certificate-threshold}
For a fixed committee and stated pointwise, workload-level, or
domain-conditional scope, assume a static Byzantine adversary with
$|\ByzantineSet|\le\ByzantineBound$. Condition on the event
$\SemanticSafetyEvent$, under which every invalid candidate under
consideration receives at most $\EpistemicBudget$ protocol-compliant
false-endorsement weight outside the Byzantine set. If
\begin{equation}
q>\ByzantineBound+\EpistemicBudget,
\label{eq:semantic-validity-threshold}
\end{equation}
then no invalid candidate can form a $q$-signature semantic certificate.
For integer thresholds, it is sufficient to choose
\begin{equation}
q\ge \left\lfloor \ByzantineBound+\EpistemicBudget \right\rfloor+1.
\label{eq:semantic-validity-rounded}
\end{equation}
This is a probabilistic statement through the event
$\SemanticSafetyEvent$ and is a sufficient condition for this certificate
rule; no tightness claim is made for protocols with additional semantic
certificate-validity evidence.
\end{theorem}

\begin{proof}
Fix an invalid candidate $x$ in context $s$. Byzantine validators contribute
at most $\ByzantineBound$ signatures because
$|\ByzantineSet|\le \ByzantineBound$. On $\SemanticSafetyEvent$,
protocol-compliant validators outside $\ByzantineSet$ contribute at most
$\EpistemicBudget$ false endorsements to any invalid candidate. Hence an
invalid candidate can collect at most
$\ByzantineBound+\EpistemicBudget$ certificate-eligible endorse signatures.
If $q>\ByzantineBound+\EpistemicBudget$, this support is insufficient to reach
the certificate threshold. Since the argument holds for every invalid
candidate in the event, no invalid candidate can form a $q$-certificate.
The rounded condition is the least integer threshold strictly larger than
$\ByzantineBound+\EpistemicBudget$.
\end{proof}

\begin{theorem}[Agreement]\label{thm:agreement-threshold}
Assume authenticated signatures, a static Byzantine set with
$|\ByzantineSet|\le\ByzantineBound$, and that only Byzantine validators may
equivocate. This result is deterministic and uses no semantic probability
space. If
\begin{equation}
2q-\NumValidators>\ByzantineBound,
\label{eq:agreement-threshold}
\end{equation}
equivalently $q>(\NumValidators+\ByzantineBound)/2$, then two conflicting
$q$-certificates cannot both form. For integer thresholds, it is sufficient to
choose
\begin{equation}
q\ge
\left\lfloor\frac{\NumValidators+\ByzantineBound}{2}\right\rfloor+1.
\label{eq:agreement-rounded}
\end{equation}
The inequality is necessary and sufficient for the stated
quorum-intersection argument; it is not claimed tight for all possible
consensus protocols.
\end{theorem}

\begin{proof}
Let $Q_1$ and $Q_2$ be the signer sets of two $q$-certificates for the same
log position. Since both are subsets of the $\NumValidators$ validators,
\[
|Q_1\cap Q_2|
= |Q_1|+|Q_2|-|Q_1\cup Q_2|
\ge 2q-\NumValidators.
\]
If $2q-\NumValidators>\ByzantineBound$, the intersection has more than
$\ByzantineBound$ validators. At most $\ByzantineBound$ validators are
Byzantine, so $Q_1\cap Q_2$ contains at least one protocol-compliant
validator. A protocol-compliant validator cannot sign two conflicting
endorsements for the same log position. Therefore the two conflicting
certificates cannot both exist.

This proof does not charge epistemically faulty but non-equivocating
validators as Byzantine. They may make incorrect semantic judgments, but under
the base model they do not create two conflicting signatures for the same log
position.
\end{proof}

\begin{theorem}[Liveness]\label{thm:liveness-threshold}
For a valid candidate inside the stated calibrated scope, condition on the
event $\LivenessEvent$ under which at most $\UnusableBudget$
protocol-compliant validator weight is unusable. Assume a static Byzantine set
with $|\ByzantineSet|\le\ByzantineBound$, partial synchrony after GST, and
allow Byzantine validators to withhold.
If
\begin{equation}
q\le \NumValidators-\ByzantineBound-\UnusableBudget,
\label{eq:liveness-threshold}
\end{equation}
then the valid candidate can obtain a $q$-signature certificate. For integer
thresholds, it is sufficient to require
\begin{equation}
q\le
\left\lfloor \NumValidators-\ByzantineBound-\UnusableBudget \right\rfloor .
\label{eq:liveness-rounded}
\end{equation}
This is a probabilistic sufficient condition through the event
$\LivenessEvent$ and is necessary for this worst-case withholding argument;
it is not claimed tight for protocols with additional recovery mechanisms.
\end{theorem}

\begin{proof}
There are at least $\NumValidators-\ByzantineBound$ protocol-compliant
validators. On $\LivenessEvent$, at most $\UnusableBudget$ of their
support is unusable because of false rejection, abstention, timeout,
divergent semantic output, or canonicalization failure. Thus at least
$\NumValidators-\ByzantineBound-\UnusableBudget$ protocol-compliant validators
can provide usable endorsements for the valid candidate. Byzantine validators
may withhold, so the proof does not rely on their signatures. After GST,
messages from usable protocol-compliant validators are eventually delivered.
If $q\le \NumValidators-\ByzantineBound-\UnusableBudget$, enough usable
endorsements eventually arrive to form a $q$-certificate.
\end{proof}

\subsection{Threshold Feasibility}

\begin{theorem}[Threshold feasibility]\label{thm:threshold-feasibility}
Assume the same static Byzantine adversary, non-equivocation, partial synchrony
for liveness, and calibrated pointwise, workload-level, or domain-conditional
scope used in Theorems~\ref{thm:semantic-certificate-threshold}--\ref{thm:liveness-threshold}.
Let
\begin{align}
q_{\min}
&=
\max\left\{
\left\lfloor \ByzantineBound+\EpistemicBudget \right\rfloor+1,\,
\left\lfloor\frac{\NumValidators+\ByzantineBound}{2}\right\rfloor+1
\right\},
\label{eq:q-min}\\
q_{\max}
&=
\left\lfloor
\NumValidators-\ByzantineBound-\UnusableBudget
\right\rfloor .
\label{eq:q-max}
\end{align}
For integral budgets, this is equivalently
\begin{align*}
q_{\min}
&=
\max\left\{
\left\lceil \ByzantineBound+\EpistemicBudget+1 \right\rceil,\,
\left\lceil\frac{\NumValidators+\ByzantineBound+1}{2}\right\rceil
\right\},\\
q_{\max}
&=
\NumValidators-\ByzantineBound-\UnusableBudget .
\end{align*}
If $q_{\min}\le q\le q_{\max}$, then the $q$-signature semantic-certificate
protocol satisfies semantic certificate validity on $\SemanticSafetyEvent$,
agreement deterministically under the non-equivocation assumption, and
liveness on $\LivenessEvent$. This interval is exact for satisfying the three
displayed threshold inequalities with one integer $q$; the later population
inequalities are sufficient unrounded design rules, not separate tightness
claims.
\end{theorem}

\begin{proof}
The lower endpoint $q_{\min}$ is the least integer threshold satisfying both
Equation~\ref{eq:semantic-validity-threshold} and
Equation~\ref{eq:agreement-threshold}. The upper endpoint $q_{\max}$ is the
largest integer threshold satisfying Equation~\ref{eq:liveness-threshold}.
Any integer $q$ between them therefore satisfies the hypotheses of
Theorems~\ref{thm:semantic-certificate-threshold},
\ref{thm:agreement-threshold}, and~\ref{thm:liveness-threshold}.
\end{proof}

Ignoring rounding, the feasible interval is
\begin{equation}
\max\left\{
\ByzantineBound+\EpistemicBudget,\,
\frac{\NumValidators+\ByzantineBound}{2}
\right\}
<
q
\le
\NumValidators-\ByzantineBound-\UnusableBudget.
\label{eq:feasible-interval}
\end{equation}
For this real-valued interval to be nonempty, the upper endpoint must exceed
both lower bounds. This gives the population conditions
\begin{align}
\NumValidators
&>
2\ByzantineBound+\EpistemicBudget+\UnusableBudget,
\label{eq:population-validity-liveness}\\
\NumValidators
&>
3\ByzantineBound+2\UnusableBudget.
\label{eq:population-agreement-liveness}
\end{align}
Equations~\ref{eq:q-min} and~\ref{eq:q-max} are the exact integer test; the
population inequalities are sufficient in the integral-count setting and show
the corresponding unrounded design rule. For real-valued weights or empirical
upper bounds, the floor and ceiling formulas above give the threshold check to
use directly.

If $\Prob[\SemanticSafetyEvent]\ge 1-\delta$ and
$\Prob[\LivenessEvent]\ge 1-\epsilon$, then both stochastic events hold
with probability at least $1-\delta-\epsilon$ by the union bound. A sharper
joint probability bound may replace this bound if the dependence between
false endorsement and unusable support is modeled directly, for example by
estimating $\Prob[\SemanticSafetyEvent\cap\LivenessEvent]$.

\noindent\textbf{How to read the guarantees.}
Agreement is deterministic under the quorum-intersection and non-equivocation
assumptions. Semantic certificate validity is conditional on
$\SemanticSafetyEvent$, liveness is conditional on $\LivenessEvent$, and the
two events hold together with probability at least $1-\delta-\epsilon$ by the
union bound unless a sharper joint model is supplied. Calibration quality determines whether these
\EBFT{} guarantees are meaningful for a deployed domain.

\begin{corollary}[Deterministic coherent and dispersive substitution]\label{cor:deterministic-substitution}
Suppose a deterministic analysis justifies substituting
$\EpistemicBudget=\CoherentFaultBound$ and
$\UnusableBudget=\CoherentFaultBound+\DispersiveFaultBound$. The unrounded
population conditions become
\begin{equation}
\NumValidators
>
\max\left\{
2\ByzantineBound+2\CoherentFaultBound+\DispersiveFaultBound,\,
3\ByzantineBound+2\CoherentFaultBound+2\DispersiveFaultBound
\right\}.
\label{eq:deterministic-substitution}
\end{equation}
This deterministic specialization is sufficient for the displayed
$q$-signature threshold test when the substituted bounds are actually proved;
it is not claimed tight for all protocols.
\end{corollary}

\subsection{Conservative Full-Byzantine Reduction}

\begin{proposition}[Conservative full-Byzantine reduction]\label{thm:confidence-quorum}
If protocol-compliant false endorsers are intentionally charged as fully
Byzantine identities for all quorum-intersection purposes, then the classical
condition
\begin{equation}
\NumValidators \ge 3\bigl(\ByzantineBound+\EpistemicBudget\bigr)+1
\label{eq:confidence-quorum}
\end{equation}
is sufficient to reuse a standard $3F+1$ BFT argument on
$\SemanticSafetyEvent$, where $F=\ByzantineBound+\EpistemicBudget$. If
$\Prob[\SemanticSafetyEvent]\ge 1-\delta$, this reduction applies with
probability at least $1-\delta$. This is a sufficient and deliberately
pessimistic reduction, not a necessary condition for the base \EBFT{}
certificate rule.
\end{proposition}

\begin{proof}
On $\SemanticSafetyEvent$, at most $\EpistemicBudget$ protocol-compliant
validators can falsely endorse a fixed invalid candidate, and at most
$\ByzantineBound$ validators are Byzantine. Treating both groups as a single
fully Byzantine set gives an effective bound
$F=\ByzantineBound+\EpistemicBudget$. The standard authenticated BFT quorum
condition $N\ge 3F+1$ is therefore sufficient for the corresponding
full-Byzantine reduction.
For integer identity counts, this reduction can be interpreted conservatively
using
$F_\delta=\ByzantineBound+\lfloor\EpistemicBudget\rfloor$; the real-valued
condition
$\NumValidators\ge 3(\ByzantineBound+\EpistemicBudget)+1$ is a stronger
sufficient design constraint.
\end{proof}

This reduction is deliberately pessimistic. It charges non-equivocating
epistemic validators as Byzantine for agreement, even though
Theorem~\ref{thm:agreement-threshold} shows that agreement depends only on
equivocation. It also does not express the liveness budget
$\UnusableBudget$. The main threshold result is therefore
Theorem~\ref{thm:threshold-feasibility}, not
Equation~\ref{eq:confidence-quorum}.

\subsection{Comparison}

Classical authenticated \BFT{} is recovered by setting
$\EpistemicBudget=\UnusableBudget=0$. The feasibility conditions reduce to
$\NumValidators>2\ByzantineBound$ and
$\NumValidators>3\ByzantineBound$, so the usual
$\NumValidators>3\ByzantineBound$ condition dominates; with
$\NumValidators=3\ByzantineBound+1$, the unique feasible threshold is
$q=2\ByzantineBound+1$.

Basilic's hybrid-fault model, written here with $b$ for benign faults to avoid
confusing it with the certificate threshold $q$, uses the bound
$n>3t+d+2b$ for consensus where $t$ replicas are Byzantine-deceitful, $d$ are
deceitful-benign (cannot equivocate but can send incorrect values), and $b$ are
benign (can only crash or omit) \cite{ranchal2022basilic}. \EBFT{} introduces a
fundamentally different classification. Epistemic faults are *purely semantic*:
validators in $\HonestSet_S$ are strictly protocol-compliant, non-equivocating, and
responsive, but may err stochastically in their semantic assessments. 

Because protocol-compliant validators never equivocate, they cannot sign different
proposals or digests for the same log position. Consequently, they pose zero threat
to protocol agreement (preventing forks). This explains why the agreement threshold
condition in Theorem~\ref{thm:agreement-threshold} ($2q - \NumValidators > \ByzantineBound$) depends solely on the Byzantine
bound $\ByzantineBound$, completely independent of the epistemic budgets $\EpistemicBudget$
and $\UnusableBudget$. In contrast, hybrid models like Basilic must charge non-equivocating
faulty replicas in their agreement or consensus bounds because those replicas can still
deviate from protocol-level state transitions. In \EBFT{}, the risk of epistemic faults is
isolated: $\EpistemicBudget$ bounds the safety risk of false endorsements (forming an invalid certificate), while $\UnusableBudget$ bounds the liveness risk of unusable support (failing to form a certificate).

Probabilistic quorum protocols such as ProBFT use probability to sample or
construct protocol quorums while preserving protocol-level safety and liveness
with high probability \cite{avelas2024probft}. The thresholds above do not
randomize message sampling. Conditional on the semantic budget events, the
certificate arguments are deterministic; probability models concentration of
semantic endorsement errors.

\section{Diversity-Aware Semantic Certification}\label{sec:diversity-protocol}

The certificate protocol in Section~\ref{sec:quorum-theory} does not require a
runtime oracle that identifies invalid semantic classes. Invalidity information
enters through offline labels, externally verifiable calibration tasks, trusted
verifiers, or other evidence used to estimate the tail budgets. At runtime, the
admission service verifies that a request falls inside the calibrated domain and
then applies the precomputed certificate threshold. It counts signed endorse
judgments for the submitted candidate; it does not estimate support for
``invalid classes'' unless an external verifier separately identifies
invalidity.

Diversity affects \EBFT{} only through measured reductions in
$\EpistemicBudget$ or $\UnusableBudget$ for the calibrated domain. Nominal
model, provider, prompt, retrieval, or tool diversity is not itself an input to
the threshold theorems.

\begin{figure*}[tbp]
\centering
\resizebox{0.86\textwidth}{!}{%
\begin{tikzpicture}[
  header/.style={font=\footnotesize\bfseries, align=center},
  rowlabel/.style={font=\footnotesize, anchor=east},
  collabel/.style={font=\scriptsize, align=center, anchor=south,
    text width=13mm},
  guide/.style={draw=gray!24, line width=0.25pt},
  marker/.style={circle, draw=gray!60, fill=gray!18,
    minimum size=3.3mm, inner sep=0pt},
  hotmarker/.style={circle, draw=red!70, fill=red!55,
    minimum size=3.5mm, inner sep=0pt},
  legend/.style={font=\footnotesize, anchor=west}
]
\def\xstep{1.28}
\def\ystep{0.58}

\node[header, anchor=east] at (-0.36,0.52) {Validators};
\node[header, anchor=south] at ({3*\xstep},0.94) {Latent Failure Domains};

\draw[gray!35, rounded corners=2pt]
  (-0.14,0.18) rectangle ({6*\xstep+0.14},{-5*\ystep-0.18});

\foreach \row/\name in {1/{$V_1$},2/{$V_2$},3/{$V_3$},4/{$V_4$},5/{$V_5$},6/{$V_6$}} {
  \node[rowlabel] at (-0.36,{-(\row-1)*\ystep}) {\name};
  \draw[guide] (0,{-(\row-1)*\ystep}) -- ({6*\xstep},{-(\row-1)*\ystep});
}

\foreach \col/\name in {
  1/{model\\lineage},
  2/{training\\data},
  3/{prompt},
  4/{retrieval},
  5/{evidence},
  6/{toolchain},
  7/{provider}
} {
  \node[collabel] at ({(\col-1)*\xstep},0.32) {\name};
  \draw[guide] ({(\col-1)*\xstep},0.12)
    -- ({(\col-1)*\xstep},{-5*\ystep-0.12});
}

\foreach \row/\col in {
  4/1,
  1/2,6/2,
  1/3,2/3,5/3,
  6/5,
  3/6,5/6,6/6,
  1/7,4/7,6/7
} {
  \node[marker] at ({(\col-1)*\xstep},{-(\row-1)*\ystep}) {};
}

\foreach \row/\col in {
  1/1,2/1,3/1,
  3/2,4/2,
  2/4,4/4,5/4,
  1/5,3/5
} {
  \node[hotmarker] at ({(\col-1)*\xstep},{-(\row-1)*\ystep}) {};
}

\node[hotmarker] at (2.40,{-5*\ystep-0.62}) {};
\node[legend] at (2.57,{-5*\ystep-0.62})
  {highlighted common-mode overlap};
\node[marker] at (2.40,{-5*\ystep-1.02}) {};
\node[legend] at (2.57,{-5*\ystep-1.02})
  {other shared dependency};
\end{tikzpicture}
}
\caption{Matrix view of validator identities and latent failure domains.
Nominally distinct validators may overlap through model lineage, training
data, prompts, retrieval sources, evidence, toolchains, or providers. Red
markers highlight one common-mode path that can concentrate false endorsement.}
\label{fig:correlation-domains}
\end{figure*}
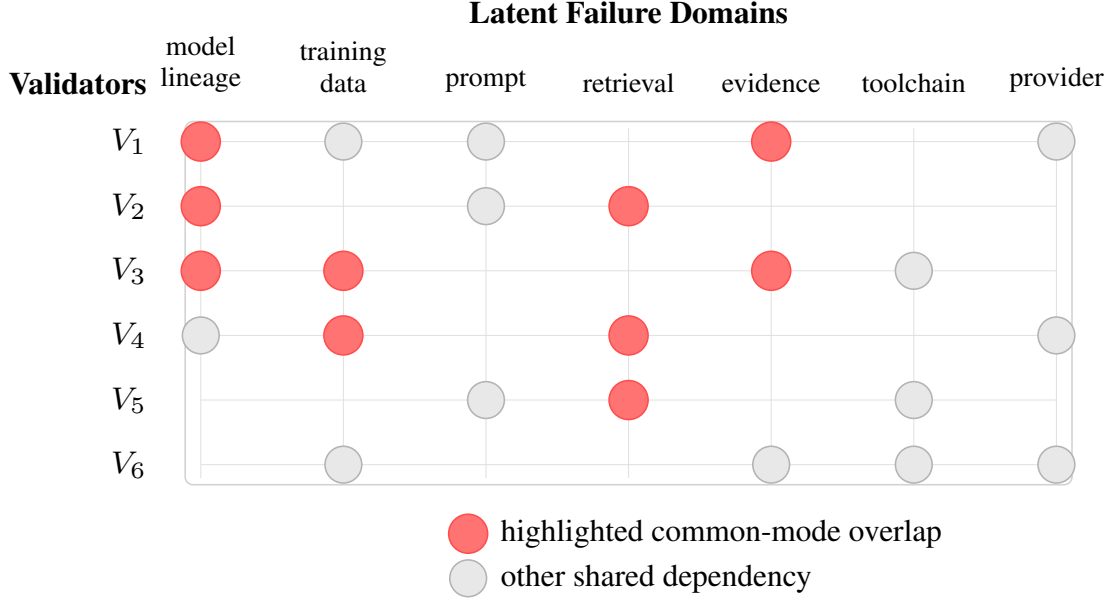

\subsection{Phase A: Offline Calibration}

Let $D$ be the workload domain for which semantic admission is to be certified.
The domain specification should identify the task family, state and policy
schema, evidence-source rules, admissible tools, expected retrieval corpora,
model and provider versions, and the kinds of transitions excluded from the
claim. Calibration evaluates candidate validators on labeled or externally
verifiable tasks drawn from $D$. The resulting record estimates
\[
\EstimatedEpistemicBudget(S,D)
\quad\text{and}\quad
\EstimatedUnusableBudget(S,D)
\]
for each candidate committee $S$, together with confidence intervals, sample
sizes, the validity domain, and an expiration date. The record also captures
model, prompt, retrieval, tool, and provider provenance, because a budget
estimated for one configuration need not transfer to another.
The true budget values are $\EpistemicBudget(S,D)$ and
$\UnusableBudget(S,D)$; $\EstimatedEpistemicBudget(S,D)$ and
$\EstimatedUnusableBudget(S,D)$ are empirical estimates, while
$\EstimatedEpistemicUpperBudget(S,D)$ and
$\EstimatedUnusableUpperBudget(S,D)$ are conservative upper endpoints used for
threshold selection.

Committee diversity changes the mathematics only through these estimated
budgets. A heterogeneous committee is useful for safety if it lowers the upper
confidence endpoint for coherent false endorsement,
$\EstimatedEpistemicUpperBudget(S,D)$; it is useful for liveness if it lowers the upper
confidence endpoint for unusable support,
$\EstimatedUnusableUpperBudget(S,D)$. A qualitative diversity score is therefore not
part of the certificate argument unless it is tied to one of these quantities.

The committee-selection step is a constrained multi-objective problem. For
latency $\ell(S)$, cost $c(S)$, capability shortfall $r_{\mathrm{cap}}(S,D)$,
and jurisdictional risk $r_{\mathrm{jur}}(S)$, one may formulate
\begin{equation}
(S^\star,q^\star)\in
\operatorname{ParetoMin}_{S\subseteq P,\,q\in\mathbb{Z}}
\Phi(S,D),
\label{eq:committee-objective}
\end{equation}
where
\[
\Phi(S,D)=
\bigl(\EstimatedEpistemicUpperBudget,\EstimatedUnusableUpperBudget,\ell,c,
r_{\mathrm{cap}},r_{\mathrm{jur}}\bigr)(S,D).
\]
The selection is subject to the threshold constraints
\begin{align}
q &> f+\EstimatedEpistemicUpperBudget(S,D),
\label{eq:committee-threshold-constraint}\\
q &> \frac{|S|+f}{2},
\nonumber\\
q &\le |S|-f-\EstimatedUnusableUpperBudget(S,D).
\nonumber
\end{align}
It must also satisfy $\ell(S)\le\ell_{\max}$, $c(S)\le c_{\max}$,
$r_{\mathrm{cap}}(S,D)=0$, and
$\mathsf{Jur}(S)\in\mathcal{J}_{\mathrm{allow}}$, where
$\mathcal{J}_{\mathrm{allow}}$ is the deployment's allowed jurisdiction policy.
The selected threshold is the integer threshold used later by the runtime
protocol. The exact integer rounding rule is the one in
Theorem~\ref{thm:threshold-feasibility}.

Calibration produces a signed committee risk profile
\begin{align*}
\mathsf{RP}=\bigl(&D,S^\star,q^\star,\delta,\epsilon,
\EstimatedEpistemicUpperBudget,\EstimatedUnusableUpperBudget,\\
&\mathsf{CI},\mathsf{expiry},\mathsf{prov},\mathsf{overlap}\bigr),
\end{align*}
where $\mathsf{CI}$ records confidence intervals, $\mathsf{prov}$ records model,
prompt, retrieval, tool, and provider provenance, and $\mathsf{overlap}$ records
the evidence-source overlap assumptions under which the estimates were made.
The profile is part of the safety case: it states the calibrated domain and
threshold assumptions before any runtime request arrives.
Operationally, the profile is a revocable, time-bound security credential. To maintain safety guarantees under potential distribution shifts or silent updates, we define the following rotation and recalibration guidelines:
\begin{itemize}
    \item \emph{Periodic Recalibration}: A risk profile must enforce a hard expiration limit ($\mathsf{expiry}$), typically set to 30 days. Re-evaluating the committee on a fresh, held-out validation workload from the domain $D$ is required to renew the profile.
    \item \emph{Event-Driven Rotation Triggers}: A profile must be immediately revoked and rotated if any of the following events occur: (i) any validator's software or underlying LLM weight version changes, (ii) the average response latency $\ell(S)$ or capability shortfall exceeds the SLA bounds, (iii) external audit detects a semantic safety violation, or (iv) the configuration of retrieval corpora or external tool schemas is updated.
    \item \emph{Emergency Fallback}: Upon profile expiration or revocation, the system falls back to a conservative threshold (e.g., the full-Byzantine reduction in Proposition~\ref{thm:confidence-quorum} with $q = \NumValidators - \ByzantineBound$) or escalates requests to a human-in-the-loop consensus path until offline recalibration concludes.
\end{itemize}

\subsection{Phase B: Runtime Admission}

A runtime request consists of a candidate transition $x$, state or context $s$,
policy $p$, and evidence package $e$. The admission service canonicalizes the
tuple into a digest
\[
h_s = h(s,x,p,\mathsf{digest}(e)).
\]
It then verifies that $h_s$ and the request metadata fall inside the signed risk
profile: the domain must match $D$, the profile must be unexpired, validator
versions and provenance must match the recorded configuration, and evidence
sources must not overlap beyond the calibrated assumption. These checks are
ordinary profile checks, not semantic invalidity checks.

If the profile checks pass, the service dispatches isolated validation packages
to the validators in $S^\star$. Each validator returns a signed structured
judgment,
\[
\langle \mathsf{endorse},\mathsf{reject},\mathsf{abstain};
\mathsf{class},h_s,\mathsf{provenance}\rangle_i.
\]
The runtime verifies signatures, provenance, version consistency, and agreement
on the canonical digest. It forms an endorsement set $Q$ only from well-formed
endorse judgments bound to the same digest and compatible semantic class. A
protocol-valid semantic certificate forms when $|Q|\ge q^\star$. The service
commits only in that case. If no such certificate forms, the request abstains or
escalates according to policy; the runtime does not infer that the candidate is
semantically invalid merely because the certificate failed.

\begin{algorithm}[tbp]
\caption{Two-phase calibrated semantic admission}
\label{alg:reference-protocol}
\footnotesize
\begin{algorithmic}[1]
\Require domain $D$, calibration tasks $T_D$, pool $P$, risks
$\delta,\epsilon$, request $(x,s,p,e)$
\Ensure $\Commit$, $\Abstain$, or $\Escalate$
\Statex \textbf{Phase A: Offline calibration}
\ForAll{candidate committees $S\subseteq P$}
  \State evaluate $S$ on labeled or externally verifiable tasks in $T_D$
  \State estimate $\widehat{e}_\delta$, $\widehat{u}_\epsilon$, intervals,
  domain limits, and expiration
  \State record model, prompt, retrieval, tool, and provider provenance
\EndFor
\State select $(S^\star,q^\star)$ satisfying threshold and operating constraints
\State sign $\mathsf{RP}\gets(D,S^\star,q^\star,\widehat{e}^{+}_\delta,
\widehat{u}^{+}_\epsilon,\mathsf{CI},\mathsf{expiry},\mathsf{prov},\mathsf{overlap})$
\Statex \textbf{Phase B: Runtime admission}
\State $h_s\gets h(s,x,p,\mathsf{digest}(e))$
\If{domain, profile, provenance, canonicalization, or overlap checks fail}
  \State \Return $\Escalate$
\EndIf
\ForAll{$i\in S^\star$}
  \State dispatch isolated package $(x,s,p,e,h_s,\mathsf{RP})$ and collect $J_i$
\EndFor
\If{too many validators are unavailable}
  \State \Return $\Abstain$
\EndIf
\State verify signatures, evidence digest, provenance, and version consistency
\State $Q\gets$ endorse judgments bound to the same $h_s$ and compatible class
\If{$|Q|\ge q^\star$}
  \State \Return $\Commit$
\ElsIf{profile checks or consistency checks fail after collection}
  \State \Return $\Escalate$
\Else
  \State \Return $\Abstain$
\EndIf
\end{algorithmic}
\end{algorithm}

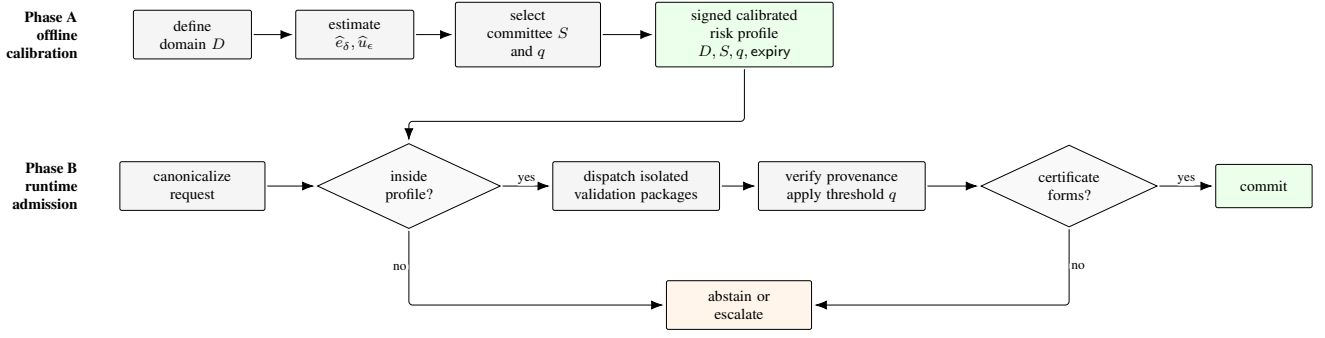
\begin{figure*}[tbp]
\centering
\resizebox{0.98\textwidth}{!}{%
\begin{tikzpicture}[
  phase/.style={font=\footnotesize\bfseries, align=right, anchor=east},
  flowstep/.style={draw, rounded corners=1pt, align=center, inner sep=4pt,
    minimum height=9mm, text width=24mm, fill=gray!8, font=\footnotesize},
  narrowstep/.style={flowstep, text width=19mm},
  widestep/.style={flowstep, text width=28mm},
  profile/.style={draw, rounded corners=1pt, align=center, inner sep=4pt,
    minimum height=9mm, text width=30mm, fill=green!8, font=\footnotesize},
  decision/.style={draw, diamond, aspect=2.15, align=center, inner sep=1pt,
    text width=19mm, fill=gray!5, font=\footnotesize},
  arrow/.style={-{Latex[length=2.2mm]}, line width=0.4pt,
    rounded corners=2pt, shorten >=1.5pt, shorten <=1pt},
  edgelabel/.style={font=\scriptsize, fill=white, inner sep=1pt}
]

\node[phase] (offline-label) at (0,0) {Phase A\\offline\\calibration};
\node[narrowstep] (domain) at (2.0,0) {define\\domain $D$};
\node[narrowstep] (calibrate) at (5.0,0) {estimate\\$\widehat{e}_\delta,\widehat{u}_\epsilon$};
\node[flowstep] (select) at (8.2,0) {select\\committee $S$\\and $q$};
\node[profile] (profile) at (12.2,0) {signed calibrated\\risk profile\\$D,S,q,\mathsf{expiry}$};

\node[phase] (runtime-label) at (0,-2.8) {Phase B\\runtime\\admission};
\node[flowstep] (canon) at (2.0,-2.8) {canonicalize\\request};
\node[decision] (check) at (6.0,-2.8) {inside\\profile?};
\node[widestep] (dispatch) at (10.2,-2.8) {dispatch isolated\\validation packages};
\node[widestep] (cert) at (14.0,-2.8) {verify provenance\\apply threshold $q$};
\node[decision] (forms) at (18.2,-2.8) {certificate\\forms?};
\node[profile, text width=15mm, minimum height=8mm] (commit) at (21.8,-2.8) {commit};

\node[flowstep, fill=orange!8, text width=24mm] (escalate) at (12.1,-5.0) {abstain or\\escalate};

\draw[arrow] (domain) -- (calibrate);
\draw[arrow] (calibrate) -- (select);
\draw[arrow] (select) -- (profile);

\draw[arrow] (profile.south) -- ++(0,-1.0) -| (check.north);

\draw[arrow] (canon) -- (check);
\draw[arrow] (check) -- node[edgelabel, above] {yes} (dispatch);
\draw[arrow] (dispatch) -- (cert);
\draw[arrow] (cert) -- (forms);
\draw[arrow] (forms) -- node[edgelabel, above] {yes} (commit);

\draw[arrow] (check.south) |- node[edgelabel, near start, left] {no} (escalate.west);
\draw[arrow] (forms.south) |- node[edgelabel, near start, right] {no} (escalate.east);

\end{tikzpicture}
}
\caption{Two-phase semantic admission. Offline calibration estimates risk
budgets and signs a risk profile; runtime admission only checks domain,
provenance, and the precomputed threshold $q$.}
\label{fig:protocol-flow}
\end{figure*}

The separation between calibration and admission is the operational counterpart
of the two-budget theory. Offline calibration estimates how much false
endorsement and unusable support a committee can concentrate over $D$; runtime
admission checks whether the signed evidence returned for one request reaches
the precomputed threshold. When the domain check fails, too many validators are
unavailable, the risk profile has expired, canonicalization is ambiguous, or
evidence-source overlap violates the calibrated assumption, the correct action
is abstention or escalation rather than an uncalibrated commit.

The next section turns these calibration obligations into an evaluation methodology.

\section{Calibration and Evaluation Methodology}\label{sec:evaluation}

This section details the methodology for estimating, stress-testing, and
falsifying \EBFT{} budgets and for evaluating whether the budgets defined in
Sections~\ref{sec:pbf-model} and~\ref{sec:quorum-theory} can guide committee
selection. It also includes an illustrative synthetic budget calculation. The
methodology is organized around four research questions.

\begin{description}
\item[RQ1] How large are $\EpistemicBudget$ and $\UnusableBudget$ for
homogeneous and heterogeneous validator committees?
\item[RQ2] Does nominal model, prompt, provider, retrieval, or tool diversity
reduce coherent false endorsement?
\item[RQ3] Do thresholds selected from the quorum conditions remain
conservative for observed unsafe-certificate and no-certificate rates?
\item[RQ4] How stable are the estimated budgets under model updates, workload
shift, and adversarial evidence?
\end{description}

The task suite contains only cases with defensible semantic-validity labels:
executable infrastructure-policy checks, configuration invariants, code or
plan verification tasks, and controlled synthetic state-transition tasks.
Labels may come from executable tests, model checking, deterministic policy
engines, trusted external verifiers, or adjudication records with auditable
evidence. Random train/test splits within one task family are insufficient;
the design reserves held-out workload families so that calibrated budgets can
be tested against distribution shift.

Each task produces a candidate transition $x$, context $s$, evidence package,
and ground-truth value for $\SemanticValidity(x,s)$. A completed study would
evaluate the same task by homogeneous committees, such as replicas of one model-prompt-tool
configuration, and heterogeneous committees that vary model family, prompt
family, provider, retrieval source, and toolchain. The primary measurements are
false endorsement, false rejection, abstention, pairwise dependence, largest
false-endorsement coalition, $\EstimatedEpistemicBudget$,
$\EstimatedUnusableBudget$,
unsafe-certificate rate, liveness-failure rate, latency, and cost.

For a committee $S$ and threshold $q$, the evaluation design compares the
\EBFT{} threshold from Theorem~\ref{thm:threshold-feasibility} against four
baselines: an independence-based baseline using validators' marginal error
rates, majority voting, weighted-confidence voting with weights fixed before
evaluation, and single-validator baselines.
A completed study reports the empirical unsafe-certificate probability on
invalid tasks and the no-certificate probability on valid tasks for each rule.
The comparison against the theorem is not a claim that the theorem predicts
exact frequencies; it asks whether thresholds chosen from upper-confidence estimates of
$\EpistemicBudget$ and $\UnusableBudget$ are conservative for observed
unsafe-certificate and no-certificate rates.

\begin{table*}[tbp]
\caption{Evaluation-configuration matrix for falsifiable evaluation. Empirical
outcomes are outside the scope of this design table.}
\label{tab:evaluation-plan}
\scriptsize
\setlength{\tabcolsep}{3pt}
\begin{tabular}{@{}L{0.07\textwidth}L{0.21\textwidth}L{0.21\textwidth}L{0.23\textwidth}L{0.10\textwidth}@{}}
\toprule
RQ & Design & Required tasks and controls & Measurements and comparisons & Falsifier \\
\midrule
RQ1 &
Run homogeneous and heterogeneous committees over labeled calibration and
held-out workload families. &
Executable policy checks, configuration invariants, code or plan verification,
and synthetic transitions with auditable labels. &
Estimate false endorsement, false rejection, abstention,
$\widehat{e}_\delta$, $\widehat{u}_\epsilon$, upper confidence endpoints,
latency, and cost with confidence intervals. &
Budgets exceed feasible thresholds. \\
RQ2 &
Ablate one diversity axis at a time: model, prompt, provider, retrieval, and
toolchain. &
Matched task sets and fixed candidate transitions; only the tested axis changes. &
Compare largest false-endorsement coalition, pairwise dependence, and budget
change against homogeneous clones. &
Nominal diversity does not reduce the upper-tail budget. \\
RQ3 &
Apply precomputed $q$ thresholds to held-out labeled tasks. &
Use committees and thresholds selected before observing held-out outcomes. &
Compare empirical unsafe-certificate and liveness-failure rates with the
theorem-bound, independence baseline, majority voting, weighted-confidence
voting, and single validators. &
Observed rates exceed stated confidence bounds. \\
RQ4 &
Repeat calibration after version changes, workload shifts, and adversarial
evidence perturbations. &
Model-version timestamps, held-out workload families, shared poisoned evidence,
ambiguous policies, common prompt injection, shared retrieval corruption, and
model-family-specific traps. &
Track drift in $\widehat{e}_\delta$, $\widehat{u}_\epsilon$,
$\widehat{e}^{+}_\delta$, $\widehat{u}^{+}_\epsilon$, coalition size,
unsafe certificates, liveness failures, latency, and cost. &
Risk profile remains accepted after budget drift. \\
\botrule
\end{tabular}
\end{table*}

Rare-event tail quantities require explicit uncertainty reporting. The report
includes the number of tasks, number of invalid and valid cases, number of
committees, resampling or concentration method, and upper confidence endpoint
used for each budget. A few hundred examples cannot support a
$10^{-6}$ unsafe-certificate claim. If the evaluation cannot resolve the
desired confidence level, it reports the weakest supported bound or states that
the calibration sample is too small for the target risk.
These intervals account for the stated sampling procedure; hidden label-system
bias or unmodeled common-mode structure requires separate assumptions.

Adversarial tests belong in the main methodology, not in a separate robustness
appendix. The test suite includes shared poisoned evidence, ambiguous policies,
common prompt injection, shared retrieval corruption, and traps targeted at one
model family. These tests separate diversity that changes the upper tail from
diversity that only changes surface provenance.

Each run records model-version timestamps, provider identifiers, prompt hashes,
retrieval corpus digests, tool versions, policy versions, and calibration
dates. A risk profile is rerun or retired when a model or tool version changes,
a retrieval corpus changes, a workload family moves outside the calibrated
domain, adversarial tests reveal a new common-mode failure, or a confidence
interval crosses a threshold-feasibility boundary.
Appendix~\ref{sec:appendix-experimental-details} gives the configuration and
artifact fields.

\subsection{Illustrative Budget Calculation}\label{sec:illustrative-calculation}

To demonstrate the concrete application of the offline calibration phase (Phase A), we present an illustrative calculation over a hypothetical workload of $1{,}000$ simulated cloud IAM policy mutations. This subsection serves as a synthetic example to show how safety and liveness budgets are computed under Theorem~\ref{thm:threshold-feasibility}, and does not report experimental results. The hypothetical workload is partitioned into $500$ valid mutations (read-only queries or scoped resource descriptions where $\SemanticValidity(x,\StateContext) = 1$) and $500$ invalid mutations (write-authority escalations, such as wildcard expansions, where $\SemanticValidity(x,\StateContext) = 0$).

The calculation compares two hypothetical validator committees with total size $\NumValidators=7$ and Byzantine bound $\ByzantineBound=1$:
\begin{itemize}
    \item A \emph{homogeneous committee} ($S_{\mathrm{homo}}$) representing seven clones of a single model-provider family with identical prompt templates.
    \item A \emph{heterogeneous committee} ($S_{\mathrm{hetero}}$) representing seven replicas distributed across three distinct model-provider families with varied prompts and retrieval sources.
\end{itemize}

For each invalid task $m \in \{1,\ldots,M_{\mathrm{invalid}}\}$, the calculation counts the hypothetical false
endorsements $F_m = F_S(x_m, s_m)$. The offline calibration phase estimates the epistemic safety budget using two complementary statistical approaches:
\begin{enumerate}
    \item \emph{Binomial Quantile Bounds}: To ensure $\Prob[F_S(x,s) > e] \le \delta$, we treat the event $\mathbf{1}[F_m > e]$ as a Bernoulli trial with success probability $p \le \delta$. Across $M_{\mathrm{invalid}}$ tasks, the number of exceedances follows $\operatorname{Binomial}(M_{\mathrm{invalid}}, p)$. The upper confidence bound $\EstimatedEpistemicUpperBudget$ at confidence level $1-\alpha = 0.95$ is the smallest integer $e$ for which the upper Clopper-Pearson confidence endpoint for $p$ is strictly less than $\delta$.
    \item \emph{Bootstrap Resampling}: To avoid parametric assumptions, we perform $B=10{,}000$ bootstrap iterations by drawing samples of size $M_{\mathrm{invalid}}$ with replacement from $\{F_m\}$. For each bootstrap sample $b$, we compute the empirical $(1-\delta)$-quantile $\widehat{e}^{(b)}_\delta$. The upper confidence endpoint $\EstimatedEpistemicUpperBudget$ is then estimated as the $(1-\alpha)$-quantile of the bootstrap replicates $\{\widehat{e}^{(b)}_\delta\}_{b=1}^B$.
\end{enumerate}
For valid tasks $m \in \{1,\ldots,M_{\mathrm{valid}}\}$, the calculation counts the unusable support $U_m = U_S(x_m, s_m)$ (rejections, abstentions, and timeouts). The liveness budget $\EstimatedUnusableBudget$ and its upper confidence bound $\EstimatedUnusableUpperBudget$ under target risk $\epsilon = 0.05$ are estimated analogously over the sample $\{U_m\}$.

Table~\ref{tab:preliminary-results} lists the assumed illustrative estimates.
The synthetic homogeneous case assigns a largest false-endorsement coalition
of five, with $\EstimatedEpistemicBudget = 4.2$ and
$\EstimatedEpistemicUpperBudget = 5$. Under
Theorem~\ref{thm:threshold-feasibility} with $\ByzantineBound=1$, safety
requires a threshold of $q > f + \EstimatedEpistemicUpperBudget = 6$ (i.e.,
$q = 7$). This threshold violates the liveness constraint
$q \le \NumValidators - f - \EstimatedUnusableUpperBudget = 7 - 1 - 2 = 4$
given $\EstimatedUnusableUpperBudget = 2$, so the committee configuration is
infeasible under these illustrative assumptions.

The synthetic heterogeneous case assigns a largest false-endorsement coalition
of two, with $\EstimatedEpistemicBudget = 1.4$ and
$\EstimatedEpistemicUpperBudget = 2$. With $\EstimatedUnusableUpperBudget = 1$
and $\ByzantineBound=1$, Theorem~\ref{thm:threshold-feasibility} gives
$q_{\min} = \max\{4, 5\} = 5$ and $q_{\max} = 5$. Thus, the heterogeneous
committee is feasible only at consensus threshold $q^\star = 5$. The example
illustrates how smaller calibrated upper endpoints attributed to diversity
would affect threshold feasibility.

\begin{table}[tbp]
\centering
\caption{Illustrative budget estimation ($\NumValidators=7$, $f=1$, risk targets $\delta=\epsilon=0.05$, $1{,}000$ tasks).}
\label{tab:preliminary-results}
\scriptsize
\begin{tabular}{lcccc}
\toprule
Committee & $\EstimatedEpistemicBudget$ & $\EstimatedEpistemicUpperBudget$ & $\EstimatedUnusableBudget$ & $\EstimatedUnusableUpperBudget$ \\
\midrule
Homogeneous ($S_{\mathrm{homo}}$) & 4.2 & 5.0 & 1.2 & 2.0 \\
Heterogeneous ($S_{\mathrm{hetero}}$) & 1.4 & 2.0 & 0.6 & 1.0 \\
\bottomrule
\end{tabular}
\end{table}

\section{Related Work}\label{sec:related-work}

\subsection{Byzantine Agreement and State Machine Replication}

Byzantine agreement formalizes consensus despite arbitrary faulty processes
\cite{pease1980reaching,lamport1982byzantine}. Partial synchrony and practical
replication protocols show how these ideas can be engineered under timing and
authentication assumptions \cite{dwork1988consensus,castro1999practical}.

\SMR{} connects consensus to replicated services by ordering commands and then
executing deterministic transitions \cite{schneider1990state}. \EBFT{} keeps
that ordering theory but asks what happens when semantic admission is supplied
by stochastic participants.

\subsection{Hybrid and Asymmetric Fault Models}

Hybrid fault models distinguish crash, omission, Byzantine, and other failure
classes so that protocols can exploit structure below full arbitrariness.
Asymmetric and trust-aware variants similarly refine the set of assumptions
under which quorum intersections are meaningful.

\EBFT{} belongs in this family, but its distinguishing object is semantic
concentration among protocol-compliant validators rather than a new network
or message-failure class.

\subsection{Quorum Systems}

Quorum systems reason about intersecting sets of participants that can certify
state or decisions. Byzantine quorum systems make the intersection argument
robust to arbitrary faulty participants \cite{malkhi1998byzantine}.

The \EBFT{} model reuses quorum reasoning only after bounding the mass that can
coherently support a single invalid semantic class. This shifts part of the
assurance argument from identity counting to distributional estimation.

\subsection{Common-Mode and Correlated Failures}

Fault-tolerant software has long recognized that replicated implementations
may share design faults or invalid assumptions. N-version programming and its
evaluation literature study when design diversity reduces common-mode errors
\cite{avizienis1985nversion,knight1986experimental}.

This manuscript imports that caution into validator committees. Different
agent identities, providers, or prompts do not by themselves prove
independence; they are hypotheses about the upper tail of coherent invalid
support.

\subsection{N-Version Programming and Design Diversity}

N-version programming proposes independent implementations as a way to reduce
software fault correlation \cite{avizienis1985nversion}. Empirical studies
showed that independence cannot be assumed merely because implementations were
separately produced \cite{knight1986experimental}.

The analogy to agent committees is direct but not complete. Agentic validators
may share training data, model architectures, retrieval corpora, or tool APIs
even when their surface identities differ.

\subsection{Probabilistic Consensus}

Randomized consensus protocols use probability to escape deterministic
impossibility or reduce expected decision time in asynchronous systems
\cite{benor1983another,fischer1985impossibility}. Their probabilities usually
concern scheduler behavior, random choices, or termination.

\EBFT{} uses probability for a different purpose: bounding semantic
concentration risk among validators whose reasoning process is stochastic.

\subsection{External Validity and Application-Aware Consensus}

Consensus protocols often include a validity predicate that restricts values
eligible for decision. In many systems this predicate is syntactic or is
checked by deterministic application logic before certification.

\EBFT{} focuses on cases where the predicate is semantic, expensive, uncertain,
or approximated by reasoning participants. The resulting assurance problem is
not only whether validators saw an admissible message, but whether their
shared semantic judgment is reliable.

\subsection{LLM Ensembles and Correlated Model Errors}

Large language model and agent ensembles are often proposed as practical
diversity mechanisms. Their value for consensus depends on whether errors are
dispersive or concentrated over semantic classes.

The empirical literature on LLM error correlation is relevant when calibrating
committee diversity and semantic concentration.

\subsection{Agentic Distributed Systems}

Agentic systems combine planning, tool use, memory, retrieval, and external
actions. When such components participate in distributed control planes, their
semantic outputs can become part of a replicated decision path.

The OpenKedge line frames this problem as governed mutation in agentic
infrastructure: model-generated intents are constrained by execution contracts,
evidence chains, and runtime admission boundaries before they are allowed to
affect operational state \cite{he2026openkedge,he2026sovereignLoops,
he2026vai,he2026sab}. \EBFT{} sits one layer below those mechanisms. It treats
agentic AI as one application of the post-deterministic model and asks how
many protocol-compliant validators may still endorse the same invalid
transition. The relevant participant need not be an LLM; it can be any
stochastic reasoning or synthesis process whose outputs are admitted by quorum.

\EBFT{} can be used to analyze admission rules for mutative cloud commands in
intent-based governance protocols. When an autonomous agent generates a
Kubernetes policy update, cloud IAM change, or Infrastructure as Code (IaC)
modification, the control plane can treat it as a declarative intent (the
candidate transition $x$). The governance protocol evaluates this intent
against safety rules (policy $p$) under the system state $s$ using credentials
(evidence $e$). A semantic certificate from a calibrated validator committee can
reduce the risk of admitting authority-escalating or boundary-violating
transitions under the stated budget and adversary assumptions; it does not by
itself prove runtime safety outside those assumptions.

\subsection{Post-Deterministic Distributed Systems}

\PDDS{} is used here for systems in which participants can be protocol
compliant while deriving transitions through nondeterministic or stochastic
semantic procedures. The term names a modeling stance, not a finished theory.

Protocol-Driven Development and the OpenKedge roadmap take a similar stance
from the software-governance side: the durable artifact is not an unconstrained
generated implementation, but a protocol, invariant set, and evidence ledger
against which generated changes are admitted \cite{he2026openkedge,
he2026pdd}. \EBFT{} gives that stance a distributed-systems fault vocabulary by
separating syntactic protocol compliance from semantic correctness under
correlated stochastic judgment.

The \EBFT{} contribution is to identify one safety-relevant quantity for such
systems: upper-tail support for a common invalid semantic class.

\subsection{Semantic Quorum Assurance}

\SQA{} refers to assurance arguments that combine quorum certificates with
evidence about semantic validity. It is the layer at which verifier quality,
canonicalization, provenance, and diversity assumptions become explicit.

Prior \SQA{} work proposes diverse sandboxed validators, risk-adaptive quorum
thresholds, and evidence-backed semantic certificates for nondeterministic AI
infrastructure \cite{he2026sqa}. The Sovereign Assurance Boundary places those
certificates at the runtime admission boundary, so that execution can depend on
evidence digests rather than model outputs alone \cite{he2026sab}. Here,
\SQA{} is operationalized through the budgets $\EpistemicBudget$ and
$\UnusableBudget$ and the admission skeleton in
Algorithm~\ref{alg:reference-protocol}; the new step is to model the
correlated invalid-support event that an \SQA{} certificate must make unlikely.

\section{Discussion and Limitations}\label{sec:discussion}

\subsection{Statistical versus Deterministic Guarantees}

The threshold theorems in Section~\ref{sec:quorum-theory} mix deterministic
protocol assumptions with statistical semantic assumptions. Agreement is the
deterministic part: under authenticated channels, partial synchrony, and the
assumption that only Byzantine validators equivocate, two conflicting
$q$-certificates cannot both form when the intersection condition holds.
Semantic certificate validity and liveness are different. They are conditioned
on the events $\SemanticSafetyEvent$ and $\LivenessEvent$: the former bounds
protocol-compliant false endorsement of invalid candidates by
$\EpistemicBudget$, and the latter bounds unusable support for a valid
candidate by $\UnusableBudget$.

The scope of these events must be stated with the certificate. A pointwise,
per-candidate claim applies to one candidate, state, policy, and evidence
package. A
domain-conditional claim applies only inside a calibrated workload domain
$D$. An average-case claim averages over a task distribution on that domain.
A uniform claim must bound every admissible task in the stated class. These
are not interchangeable guarantees. In particular, the paper's statistical
claims do not turn semantic correctness into a deterministic property of
consensus.

\subsection{Workload Selection and Distribution Shift}

The budgets $\EpistemicBudget(S,D)$ and $\UnusableBudget(S,D)$ are meaningful
only relative to the workload and sampling convention used to define them.
If calibration estimates an average-case $\EpistemicBudget$ over routine
infrastructure changes, an adversary who chooses inputs after seeing the
calibration can steer the system toward rare or excluded cases where the
average does not apply. The admission protocol therefore has to treat the
domain check, risk-profile expiration, and escalation paths in
Section~\ref{sec:diversity-protocol} as part of the theorem assumptions, not
as implementation detail.

Distribution shift can come from the workload, the state representation, the
policy language, or the available evidence. A committee calibrated on
well-formed policy mutations may have a different false-endorsement tail on
ambiguous rollbacks, partial migrations, or adversarially phrased evidence.
Both $\EpistemicBudget$ and $\UnusableBudget$ need expiration and recalibration
after a model update, prompt change, retrieval change, toolchain change, or
policy change. Held-out workload families and adversarial tests, as described in
Section~\ref{sec:evaluation}, are therefore evidence about a stated domain;
they are not evidence for all future tasks.

\subsection{Validity Oracles and Label Uncertainty}

Calibration separates two error sources. Model error is the actual stochastic
behavior of validators: false endorsement, false rejection, abstention,
timeout, divergent output, and other unusable responses. Estimation error is
uncertainty about the measured budgets $\EstimatedEpistemicBudget$ and
$\EstimatedUnusableBudget$ caused by finite samples, label noise, rare-event
uncertainty, and imperfect external checks. A low empirical error rate does
not by itself establish a small upper-tail budget.

The strongest setting is a deterministic validity oracle or executable
external verifier for the calibrated task class. Many semantic domains only
provide labels through audits, simulations, proofs with assumptions, or expert
judgment. In those cases the risk profile records label provenance and
uncertainty separately from validator provenance. The theorem can consume a
defensible bound on invalid endorsement; it does not create that bound.

\subsection{Hidden Common Failure Domains}

Provider diversity is not proof of failure independence. Validators from
different providers may share training corpora, benchmarks, post-training
methods, libraries, infrastructure dependencies, or incentives that make their
semantic failures correlated. Even nominally independent implementations can
share the same policy examples, parser assumptions, retrieval indexes, or
tool wrappers.

These hidden common domains matter because \EBFT{} is about upper-tail
concentration, not just marginal error. A diverse committee can reduce
$\EpistemicBudget$ or $\UnusableBudget$ only if it reduces the probability or
weight of coherent failure under the calibrated workload. Diversity labels are
therefore evidence to be tested, not assumptions that imply independence.

Budget estimation is harder with proprietary, black-box LLM APIs. Operators
usually cannot audit model lineage, training corpora, or fine-tuning data, so
structural checks for common-mode dependence may be unavailable. Calibration
uses bounded pessimism when setting the upper confidence
endpoint $\EstimatedEpistemicUpperBudget$: instead of assuming independence or
relying only on empirical averages, the risk profile reserves mass for
correlated validator errors. This can be modeled by adjusting the joint error
distribution or by applying distribution-free concentration inequalities, such
as Chebyshev or Markov variants for dependent variables. The point is not to
prove independence among black-box systems, but to keep the threshold
calculation conservative when latent lineage is unknown.
Opaque provider-side updates should be treated as provenance drift unless
versions are contractually pinned or the change is externally detected and
recalibrated.

\subsection{Canonicalization and Evidence Correlation}

The certificate protocol assumes that all validators receive the same
canonical state, candidate transition, policy, and evidence digest. That
assumption can fail in ways that dominate model diversity. Shared evidence can
dominate model diversity: a shared poisoned source, stale snapshot, ambiguous
policy reference, or correlated tool failure can synchronize validators that
would otherwise reason differently.

Canonicalization is also semantic. If distinct transitions are collapsed into
one digest or semantic class, the protocol may overstate agreement about the
same object. If equivalent transitions are split, it may understate support
or create avoidable liveness failures. Runtime escalation on ambiguous
canonicalization is therefore a safety condition: consensus cannot compensate
for disagreement about what was validated.

Cloud infrastructure intents illustrate the issue. Suppose two models
authorize access to the same storage bucket. The first outputs $x_1$:
\texttt{\{"Version": "2012-10-17",}\allowbreak\
\texttt{"Statement": [\{"Effect": "Allow",}\allowbreak\
\texttt{"Action": "s3:GetObject",}\allowbreak\
\texttt{"Resource":}\allowbreak\
\texttt{"arn:aws:s3:::}\allowbreak\texttt{my-bucket/}\allowbreak\texttt{*"\}]\}}.
The second outputs $x_2$, representing the same permission with different key
ordering and array formatting:
\texttt{\{"Statement": [\{"Resource":}\allowbreak\
\texttt{"arn:aws:s3:::my-bucket/*",}\allowbreak\
\texttt{"Effect": "Allow",}\allowbreak\
\texttt{"Action": ["s3:GetObject"]\}]\}}.
If $\pi_s$ is purely syntactic, such as a raw string hash, the digests
$h(s, x_1, p, \mathsf{digest}(e))$ and
$h(s, x_2, p, \mathsf{digest}(e))$ diverge. Replicas then bind endorsements to
different digests, splitting committee support and causing a liveness failure
despite semantic agreement. A suitable $\pi_s$ parses the JSON, sorts keys, normalizes array syntax, and maps equivalent policies to a canonical representation before computing the digest.

However, syntactic canonicalization does not resolve deeper semantic parsing differences. In AWS IAM policies, wildcard matching rules (e.g., \texttt{"arn:aws:s3:::my-bucket/*"}) can be parsed differently by different validators, depending on the SDK version, parser limits, or policy variable evaluation (e.g., resolving \texttt{\$\{aws:username\}}). If one validator treats a wildcard as allowing recursive directory access while another uses strict non-recursive matching, they may evaluate different safety boundaries. If the canonicalization function $\pi_s$ collapses these distinct semantic interpretations into the same digest, the protocol might aggregate endorsements from validators that disagree on the actual security implications, violating safety. Thus, $\pi_s$ must either canonicalize at the semantic boundary of policy evaluation or the system must enforce strict parser homogeneity across all validator sandboxes.

\subsection{Adaptive Adversaries}

The first threshold results use a static Byzantine set and fixed committee
risk profile. Stronger adversaries require stronger assumptions. An adaptive
adversary may choose tasks after calibration, target validators whose
provenance is known, exploit a provider update, or corrupt evidence sources
shared by many validators. Such behavior can invalidate an average-case
$\EpistemicBudget$ even when the original calibration was honestly estimated.

Specifically, if an adversary can dynamically corrupt or influence protocol-compliant validators (e.g., via targeted prompt injection or exploiting model weight leakage), they can force coherent false endorsements. \EBFT{} models this by transitioning the safety budget from a workload-average $\EpistemicBudget(S,D)$ to a uniform, worst-case bound over the adversary's action space, or by increasing the Byzantine bound $\ByzantineBound$ to include dynamically influenced validators. 

Lineage poisoning represents a related threat, where a provider silently updates model weights, introducing hidden correlated failures. Provenance tracking mitigates this by recording cryptographic hashes of model weights or API version tags. However, to prevent silent drift, the system must continuously run canary verification suites on the committee. If a silent update changes a model's latent behavior, the provenance check marks the risk profile as expired, preventing the committee from certifying transitions until a new offline calibration is completed.

\subsection{Cost of Epistemic Diversity}

Reducing coherent semantic failure is not free. Larger or more heterogeneous
committees can increase latency, cost, operational complexity, and the amount
of sensitive state disclosed to external validators. Jurisdictional rules may
exclude otherwise useful validators. Capability constraints may also increase
$\UnusableBudget$ if diverse validators lack the tools, context window,
policy language support, or domain expertise needed to return usable
judgments.

Committee selection is consequently a constrained engineering problem, not a
search for maximum variety. Diversity is valuable when it improves the
estimated safety or liveness budgets enough to satisfy the certificate
thresholds under acceptable cost, latency, privacy, and governance limits.

\subsection{Performance and Latency Overheads}

An \EBFT{} interception layer changes the latency profile of admission. In
classical State Machine Replication (\SMR{}), replicas validate deterministic
transition logic and can often complete ordering and commit on a millisecond
fast path. Phase B validation may require model calls, external verifiers,
structured judgment collection, and evidence checks. Large language model
inference can take hundreds of milliseconds to several seconds, so semantic
certification is unlikely to fit on the same synchronous commit path as
ordinary deterministic validation.

Deployment patterns may reduce perceived overhead, but each changes the
failure model and needs its own timeout, rollback, and side-effect policy:
\begin{itemize}
    \item \emph{Off-Critical-Path Validation}: Replicas order proposals via fast-path consensus, but defer final execution until the asynchronous validation phase collects $|Q| \ge q^\star$ signatures.
    \item \emph{Optimistic Execution}: Replicas speculatively execute transitions inside a containerized sandbox with a copy-on-write (CoW) filesystem overlay and transactional database snapshots. If the semantic certificate fails to form within a timeout window, the speculative state is rolled back by discarding the CoW overlay and aborting the database transaction. Crucially, any external side effects (e.g., outbound API calls or physical mutations) are buffered in a secure outbound queue and released only when the certificate commits.
    \item \emph{Semantic Caching}: Replicas cache pre-signed certificates for frequent or routine mutations (e.g., recurring read-only query permissions) at the gateway layer, avoiding runtime inference. Cache keys are indexed by the canonical digest $h_s$. A cached certificate is invalidated immediately under three conditions: (i) the policy $p$ or state schema $s$ changes, (ii) the risk profile $\mathsf{RP}$ expires or is rotated, or (iii) any dependent evidence source in $e$ drifts or changes, ensuring stale semantic judgments are never reused.
\end{itemize}

\subsection{Scope of the EBFT Abstraction}

\EBFT{} does not say that consensus produces semantic correctness. Classical
agreement orders values and prevents conflicting certificates under the stated
fault assumptions; it does not decide whether an admitted infrastructure
transition is semantically safe. Semantic correctness must come from a
validity oracle, externally grounded evidence, a deterministic verifier, or a
calibrated probabilistic bound on protocol-compliant endorsement errors.

The \EBFT{} guarantee is narrower: a quantified certificate risk under
explicit probabilistic and protocol assumptions. If the
committee, workload domain, semantic-validity evidence, canonicalization procedure,
adversary model, expiration policy, and confidence-indexed budgets are valid,
then the threshold theorems translate those assumptions into certificate
conditions for semantic certificate validity, agreement, and liveness. When any of those
assumptions changes, the risk profile is revisited before the guarantee
is reused.

\section{Conclusion}\label{sec:conclusion}

Agentic infrastructure creates a fault mode in which protocol compliance and
semantic correctness diverge. A validator may authenticate, respond on time,
sign the expected message, and avoid equivocation while still producing an
incorrect semantic judgment about an operational state transition.

The Honest Quorum Problem occurs when protocol-compliant reasoning validators
coherently endorse an invalid transition strongly enough to form a certificate.
\EBFT{} addresses that problem by adding confidence-indexed semantic budgets to
the conventional Byzantine fault bound: $e_\delta$ for coherent invalid
endorsement that threatens semantic certificate validity, and $u_\epsilon$ for
unusable support that threatens liveness.

Nominal agent count or provider diversity is not sufficient evidence of fault
tolerance. Classical \BFT{} bounds arbitrary faulty identities; \EBFT{}
additionally bounds how much protocol-compliant support may concentrate on an
invalid semantic judgment. That is the core lesson of the Honest Quorum
Problem: agreement can be real, well signed, and protocol valid while still
failing to establish semantic validity.

\begin{appendices}
\section{Formal Proofs}\label{sec:appendix-proofs}

\subsection{Preliminaries and Quorum-Intersection Lemma}

All proofs in this appendix use the base semantic-certificate model unless a
statement says otherwise. The validator set is
$\ValidatorSet=\{1,\ldots,\NumValidators\}$, the Byzantine set is
$B\subseteq\ValidatorSet$ with $|B|\le \ByzantineBound$, and
$H=\ValidatorSet\setminus B$. The base model is equal-weight: $\NumValidators$,
$\ByzantineBound$, and the certificate threshold $q$ are identity counts, and
$q$ is an integer. The budgets $\EpistemicBudget$ and $\UnusableBudget$ are
nonnegative upper bounds measured in the same units of support. In the
equal-weight model, realized endorsement and unusable-support weights are
integer counts, while empirical or confidence-indexed budget bounds may be
real numbers.

The probability space is the semantic-validation probability space from
Section~\ref{sec:system-model}: workload or task sampling, validator-local
randomness, retrieval and tool randomness, model versions, and latent
common-mode variables. Byzantine validators are not sampled from this process;
their behavior is worst-case. Unless stated otherwise, the adversary is
static: the set $B$ is fixed for the execution, although Byzantine validators
may coordinate, equivocate, and withhold. The semantic-safety event is
$\SemanticSafetyEvent$, on which protocol-compliant false endorsement of each
invalid candidate under consideration is at most $\EpistemicBudget$. The
liveness event is $\LivenessEvent$, on which unusable
protocol-compliant support for the valid candidate is at most
$\UnusableBudget$.

\begin{lemma}[Quorum intersection]\label{lem:app-quorum-intersection}
Let $Q_1,Q_2\subseteq\ValidatorSet$ be two signer sets with
$|Q_1|\ge q$ and $|Q_2|\ge q$, where all variables are equal-weight counts and
$q$ is an integer. Then
\[
  |Q_1\cap Q_2|\ge 2q-\NumValidators .
\]
Consequently, every two $q$-quorums intersect in more than
$\ByzantineBound$ validators if and only if
$2q-\NumValidators>\ByzantineBound$, equivalently
$q>(\NumValidators+\ByzantineBound)/2$. The least integer threshold satisfying
this inequality is
\[
  q\ge
  \left\lfloor\frac{\NumValidators+\ByzantineBound}{2}\right\rfloor+1 .
\]
This lemma is deterministic and uses no probability space. It is an exact
set-intersection fact, not a claim that a full consensus protocol has a
matching lower bound.
\end{lemma}

\begin{proof}
The inclusion-exclusion identity gives
\[
|Q_1\cap Q_2|
= |Q_1|+|Q_2|-|Q_1\cup Q_2|
\ge q+q-\NumValidators .
\]
If $2q-\NumValidators>\ByzantineBound$, then the intersection contains more
than $\ByzantineBound$ identities. Conversely, the smallest possible
intersection of two $q$-subsets of an $\NumValidators$-element set is
$\max\{0,2q-\NumValidators\}$. Since $\ByzantineBound\ge 0$, if
$2q-\NumValidators\le\ByzantineBound$ there exist two $q$-subsets whose
intersection has at most $\ByzantineBound$ identities. The integer formula is
the least integer strictly larger than
$(\NumValidators+\ByzantineBound)/2$.
\end{proof}

\subsection{Proof of Agreement/Semantic-Validity Separation}

\begin{theorem}[Agreement does not imply semantic certificate validity]
\label{thm:app-agreement-not-validity}
Assume the equal-weight base certificate model with integer threshold $q$,
authenticated channels, and a static Byzantine adversary. Let the probability
space be degenerate except for protocol-compliant semantic judgments, and let
the chosen validator-output distribution put probability one on the judgments
described below. There exists a legal execution in which all messages are
authenticated, no protocol-compliant validator equivocates, the consensus
protocol satisfies agreement and termination, and an invalid transition
obtains a certificate. This is an existence separation; it is not a necessary
condition or a tight threshold result.
\end{theorem}

\begin{proof}
Choose an admission instance $(x,s,h_s)$ with
$\SemanticValidity(x,s)=0$. Let $B=\emptyset$, so every validator is
protocol-compliant and the adversary has no Byzantine identities. Choose a
set $Q\subseteq H$ with $|Q|=q$. In the realized validator-output distribution,
each validator $i\in Q$ returns exactly one authenticated judgment:
\[
  J_i(x,s)=\mathsf{endorse},
\]
bound to the same view, candidate digest, context identifier,
evidence-package identity, and semantic class. Validators outside $Q$ may
abstain. All messages are signed by their senders, and no validator in $H$
signs two conflicting judgments for the same admission instance.

The certificate rule is count-based and accepts $q$ distinct signed endorse
judgments for the same bound instance, so $Q$ forms a conventional externally
valid semantic certificate. Run any consensus protocol that satisfies agreement
and termination and whose external validity rule accepts such a well-formed
certificate. In a timely legal execution, termination gives a decision, and
agreement ensures that protocol-compliant replicas decide the same certified
transition.
The decided transition is nevertheless semantically invalid because
$\SemanticValidity(x,s)=0$. Thus protocol agreement and termination can hold
while semantic certificate validity fails.
\end{proof}

\subsection{Count-Only Impossibility Proof}

\begin{theorem}[No count-only semantic guarantee without concentration]
\label{thm:app-count-only-impossible}
Let $\Gamma$ be any nontrivial certificate rule whose decision depends only on
signed validator identities and protocol-visible judgments. The variables used
by $\Gamma$ may be identity counts or fixed validator weights, but the rule
does not observe a sound semantic-validity oracle. Assume a static adversary
and an unrestricted validator-output probability space for
protocol-compliant semantic errors. Then there exists an invalid admission
instance and a validator-output distribution under which an invalid candidate
satisfies $\Gamma$ with probability one. The result shows that an additional
semantic-grounding assumption is necessary for semantic certificate validity;
it is not a numeric tightness claim.
\end{theorem}

\begin{proof}
Because $\Gamma$ is nontrivial, some protocol-visible transcript
$T^\star$ is accepted. If no transcript were accepted, $\Gamma$ would not be a
live certificate rule. The transcript contains only fields visible to the
protocol: validator identities or weights, signatures, views or rounds,
candidate digests, evidence-package identities, and structured judgments.

Construct two specification-level interpretations compatible with the same
protocol-visible transcript. In
$E^+$, the candidate is semantically valid:
$\SemanticValidity(x,s^+)=1$, and the validators whose signed judgments appear
in $T^\star$ return exactly the judgments shown in $T^\star$. In $E^-$, the
same visible identities, signatures, view, candidate digest,
evidence-package identity, and judgments occur, but the semantic state
corresponding to the admission instance is invalid:
$\SemanticValidity(x,s^-)=0$. This is possible precisely because
$\Gamma$ does not receive a deterministic validity oracle, a trusted verifier,
or independently sufficient semantic evidence that rules out one of the
compatible interpretations.

The two executions are indistinguishable to $\Gamma$: by assumption,
$\Gamma$ is a function only of the protocol-visible transcript, and that
transcript is $T^\star$ in both executions. Hence $\Gamma$ accepts in
$E^-$ whenever it accepts in $E^+$.

It remains to realize $E^-$ probabilistically. Choose the
protocol-compliant validator-output distribution concentrated on the event
that every protocol-compliant validator whose judgment appears in
$T^\star$ returns the signed endorsement shown in $T^\star$ for the invalid
instance, while all other validators return any judgments consistent with the
same transcript. The distribution is allowed because semantic errors outside
the Byzantine set are unrestricted. Therefore the invalid candidate satisfies
$\Gamma$ with probability one. No certificate rule based only on counts,
weights, identities, and visible judgments can guarantee semantic certificate
validity without an additional concentration or semantic-grounding assumption.
\end{proof}

\subsection{Common-Mode Replication Theorem}

\begin{theorem}[Common-mode replication floor]
\label{thm:app-common-mode-floor}
Assume a fixed committee with a static Byzantine set and a probability space
that includes a latent common-mode event $\CommonModeEvent$. Let
$G\subseteq H$ be a replicated protocol-compliant model family. In the
equal-weight case, the support of $G$ is $|G|$ identities; in a weighted
variant, it is the total weight of $G$. Suppose that, on $\CommonModeEvent$, every
validator in $G$ falsely endorses the same invalid candidate for the same
context, evidence package, and semantic class. If the threshold rule is
satisfied by the endorsements of $G$ alone, then the unsafe-certificate
probability is at least $\Prob[\CommonModeEvent]$. This is a lower bound on risk,
not a probabilistic upper bound and not a tight characterization unless all
other unsafe events are ruled out.
\end{theorem}

\begin{proof}
Let $A$ be the event that an invalid candidate obtains a certificate accepted
by the threshold rule. By assumption, when $\CommonModeEvent$ occurs, every member
of $G$ endorses the same invalid candidate and the endorsements of $G$ satisfy
the threshold rule. Hence $\CommonModeEvent\subseteq A$. Monotonicity of probability
therefore gives
\[
  \Prob[A]\ge \Prob[\CommonModeEvent].
\]
Adding more validators from the same replicated family may change the
conditional variance of endorsements outside $\CommonModeEvent$, but it does not
reduce the probability of $\CommonModeEvent$ or the implication
$\CommonModeEvent\subseteq A$ so long as the cloned group still satisfies the
threshold. The common-mode event is therefore a replication floor on
unsafe-certificate probability.
\end{proof}

\subsection{Semantic-Certificate Validity Theorem}

\begin{theorem}[Semantic certificate validity]
\label{thm:app-semantic-certificate-validity}
Assume the equal-weight base certificate model, the semantic-validation
probability space from the preliminaries, a static Byzantine set $B$ with
$|B|\le\ByzantineBound$, integer threshold $q$, authenticated signatures, and
worst-case Byzantine endorsements. Let
$\EpistemicBudget$ be a nonnegative count-valued or real-valued upper bound,
in units of validator support, on protocol-compliant false endorsement outside
$B$. On the event $\SemanticSafetyEvent$, every invalid candidate under
consideration receives at most $\EpistemicBudget$ false endorsements from
$H$. If
\[
  q>\ByzantineBound+\EpistemicBudget,
\]
equivalently for integer thresholds
\[
  q\ge \left\lfloor\ByzantineBound+\EpistemicBudget\right\rfloor+1,
\]
then no invalid candidate can form a $q$-signature semantic certificate on
$\SemanticSafetyEvent$. If
$\Prob[\SemanticSafetyEvent]\ge 1-\delta$, the semantic-certificate-validity
claim holds with probability at least $1-\delta$. The condition is sufficient;
for this certificate rule without additional semantic certificate-validity
evidence, a threshold not exceeding the possible invalid support cannot rule
out an invalid certificate, but no broader tightness claim is made.
\end{theorem}

\begin{proof}
Fix any invalid candidate $x$ in context $s$. Byzantine validators can
contribute at most $\ByzantineBound$ certificate-eligible signatures because
$|B|\le\ByzantineBound$. On $\SemanticSafetyEvent$, protocol-compliant
validators outside $B$ contribute at most $\EpistemicBudget$ false
endorsements to this invalid candidate. Thus the total support available to
an invalid certificate is at most
$\ByzantineBound+\EpistemicBudget$. Since certificate support is counted in
integer signatures, any integer threshold
$q\ge\lfloor\ByzantineBound+\EpistemicBudget\rfloor+1$ is strictly larger
than that maximum support. The invalid candidate therefore cannot reach the
threshold. The argument holds for every invalid candidate covered by
$\SemanticSafetyEvent$, so the property holds throughout that event. If the
event has probability at least $1-\delta$, the same probability lower bound
applies to the conditional semantic-certificate-validity guarantee.
\end{proof}

\subsection{Agreement Theorem}

\begin{theorem}[Agreement]
\label{thm:app-agreement-threshold}
Assume the equal-weight base certificate model with $\NumValidators$ validator
identities, integer threshold $q$, authenticated signatures, and a static
Byzantine set $B$ with $|B|\le\ByzantineBound$. Only Byzantine validators may
equivocate; protocol-compliant validators may make semantic mistakes but sign
at most one endorse judgment for a log position. The probability space is not
used; the claim is deterministic. If
\[
  2q-\NumValidators>\ByzantineBound,
\]
equivalently
\[
  q\ge
  \left\lfloor\frac{\NumValidators+\ByzantineBound}{2}\right\rfloor+1,
\]
then two conflicting $q$-certificates for the same log position cannot both
form. The inequality is necessary and sufficient for the stated
quorum-intersection property; this does not assert a tight lower bound for
all consensus protocols.
\end{theorem}

\begin{proof}
Let $Q_1$ and $Q_2$ be the signer sets of two alleged conflicting
$q$-certificates for the same log position. By
Lemma~\ref{lem:app-quorum-intersection},
$|Q_1\cap Q_2|\ge 2q-\NumValidators$. If
$2q-\NumValidators>\ByzantineBound$, then the intersection contains more than
$\ByzantineBound$ validators. Since at most $\ByzantineBound$ validators are
Byzantine, at least one validator in the intersection is
protocol-compliant. That validator would have signed two conflicting endorse
judgments for the same log position, contradicting the base-model
non-equivocation rule. Thus the two conflicting certificates cannot both
exist.

If the intersection inequality fails, the set-intersection argument alone
cannot exclude two $q$-signer sets whose overlap is contained entirely within
a Byzantine set of size at most $\ByzantineBound$. Byzantine validators in
that overlap could equivocate. This establishes necessity for the
intersection condition used by this certificate rule, not a global lower
bound for every possible protocol.
\end{proof}

\subsection{Liveness Theorem}

\begin{theorem}[Liveness]
\label{thm:app-liveness-threshold}
Assume the equal-weight base certificate model, the semantic-validation
probability space from the preliminaries, integer threshold $q$, a static
Byzantine set $B$ satisfying $|B|\le\ByzantineBound$, authenticated channels,
and partial synchrony after GST. Byzantine validators may withhold. For a
semantically valid candidate with a consistent evidence package, let
$\UnusableBudget$ be a count-valued or real-valued upper bound, in units of
validator support, on unusable protocol-compliant support. On the event
$\LivenessEvent$, at most $\UnusableBudget$ support in $H$ is unusable
because of false rejection, abstention, timeout, divergent semantic output, or
canonicalization failure. If
\[
  q\le \NumValidators-\ByzantineBound-\UnusableBudget,
\]
equivalently for integer thresholds
\[
  q\le
  \left\lfloor
    \NumValidators-\ByzantineBound-\UnusableBudget
  \right\rfloor,
\]
then the valid candidate can obtain a $q$-signature certificate on
$\LivenessEvent$. If
$\Prob[\LivenessEvent]\ge 1-\epsilon$, the liveness claim holds with
probability at least $1-\epsilon$. The condition is sufficient and is necessary
for this worst-case withholding argument, but it is not a tight lower bound
for protocols with additional recovery mechanisms.
\end{theorem}

\begin{proof}
There are at least
$\NumValidators-\ByzantineBound$ protocol-compliant validators. On
$\LivenessEvent$, at most $\UnusableBudget$ of their support is
unusable. Therefore at least
$\NumValidators-\ByzantineBound-\UnusableBudget$ units of
protocol-compliant support can provide usable endorsements for the valid
candidate. Byzantine validators may withhold, so the proof does not rely on
their signatures. After GST, messages from usable protocol-compliant validators
are eventually delivered. If the integer threshold $q$ is at most the floor of
the guaranteed usable support, enough endorsements eventually arrive to form a
certificate.

If $q$ exceeds this guaranteed usable support, a worst-case execution can have
all Byzantine validators withhold and exactly the permitted amount of
protocol-compliant support be unusable. The proof assumptions then guarantee
fewer than $q$ usable signatures, so this liveness argument cannot succeed.
\end{proof}

\subsection{Threshold-Feasibility Corollary}

\begin{corollary}[Threshold feasibility]
\label{cor:app-threshold-feasibility}
Assume the equal-weight base certificate model, the semantic-validation
probability space from the preliminaries, integer threshold $q$, static
Byzantine set with $|B|\le\ByzantineBound$, authenticated channels, partial
synchrony for liveness, non-equivocation outside $B$, and the stochastic
events $\SemanticSafetyEvent$ and $\LivenessEvent$ defined above. Let
\[
q_{\min}=
\max\left\{
\left\lfloor\ByzantineBound+\EpistemicBudget\right\rfloor+1,\,
\left\lfloor\frac{\NumValidators+\ByzantineBound}{2}\right\rfloor+1
\right\}
\]
and
\[
q_{\max}=
\left\lfloor
\NumValidators-\ByzantineBound-\UnusableBudget
\right\rfloor .
\]
If $q_{\min}\le q\le q_{\max}$, the same $q$-signature protocol satisfies
semantic certificate validity on $\SemanticSafetyEvent$, agreement
deterministically, and liveness on $\LivenessEvent$. If
$\Prob[\SemanticSafetyEvent]\ge 1-\delta$ and
$\Prob[\LivenessEvent]\ge 1-\epsilon$, the two stochastic events hold
together with probability at least $1-\delta-\epsilon$ by the union bound. The
integer interval condition is necessary and sufficient for satisfying these
three displayed inequalities with one $q$; the population inequalities below
are sufficient unrounded design rules and are not claimed tight.
\end{corollary}

\begin{proof}
The lower endpoint $q_{\min}$ is the least integer threshold satisfying both
the semantic-certificate-validity inequality
$q>\ByzantineBound+\EpistemicBudget$ and the agreement inequality
$q>(\NumValidators+\ByzantineBound)/2$. The upper endpoint $q_{\max}$ is the
largest integer threshold satisfying the liveness inequality
$q\le \NumValidators-\ByzantineBound-\UnusableBudget$. Therefore any integer
$q$ in the interval satisfies the hypotheses of
Theorems~\ref{thm:app-semantic-certificate-validity},
\ref{thm:app-agreement-threshold}, and~\ref{thm:app-liveness-threshold}.
Conversely, no integer outside the interval satisfies all three inequalities
as stated.

Ignoring rounding, the interval is nonempty if its upper endpoint exceeds both
lower endpoints:
\[
\NumValidators-\ByzantineBound-\UnusableBudget
>
\ByzantineBound+\EpistemicBudget
\]
and
\[
\NumValidators-\ByzantineBound-\UnusableBudget
>
\frac{\NumValidators+\ByzantineBound}{2}.
\]
These rearrange to
\[
\NumValidators
>
2\ByzantineBound+\EpistemicBudget+\UnusableBudget
\]
and
\[
\NumValidators
>
3\ByzantineBound+2\UnusableBudget .
\]
Finally, the union bound gives
$\Prob[\SemanticSafetyEvent\cap\LivenessEvent]
\ge 1-\delta-\epsilon$. A sharper joint probability statement would require a
separate model of dependence between the two events.
\end{proof}

\subsection{Deterministic \texorpdfstring{$e_c/e_i$}{ec/ei} Specialization}

\begin{corollary}[Deterministic coherent and dispersive specialization]
\label{cor:app-deterministic-substitution}
Assume the equal-weight base certificate model and static Byzantine set from
Corollary~\ref{cor:app-threshold-feasibility}. Suppose a deterministic
analysis, not a statistical estimate, establishes count bounds
$\CoherentFaultBound$ and $\DispersiveFaultBound$ such that
$\EpistemicBudget=\CoherentFaultBound$ and
$\UnusableBudget=\CoherentFaultBound+\DispersiveFaultBound$. The probability
space is then degenerate, so these are deterministic bounds. The exact integer threshold test is obtained by
substituting these values into $q_{\min}$ and $q_{\max}$. The unrounded
sufficient population condition is
\[
\NumValidators
>
\max\left\{
2\ByzantineBound+2\CoherentFaultBound+\DispersiveFaultBound,\,
3\ByzantineBound+2\CoherentFaultBound+2\DispersiveFaultBound
\right\}.
\]
This specialization is a sufficient condition only; no tightness is asserted
without a matching lower bound and a protocol achieving it.
\end{corollary}

\begin{proof}
Substitute
$\EpistemicBudget=\CoherentFaultBound$ and
$\UnusableBudget=\CoherentFaultBound+\DispersiveFaultBound$ into the two
unrounded population inequalities in
Corollary~\ref{cor:app-threshold-feasibility}. The first becomes
\[
\NumValidators
>
2\ByzantineBound+\CoherentFaultBound+
(\CoherentFaultBound+\DispersiveFaultBound)
=
2\ByzantineBound+2\CoherentFaultBound+\DispersiveFaultBound .
\]
The second becomes
\[
\NumValidators
>
3\ByzantineBound+2(\CoherentFaultBound+\DispersiveFaultBound)
=
3\ByzantineBound+2\CoherentFaultBound+2\DispersiveFaultBound .
\]
Taking the maximum gives the displayed sufficient condition. The exact
integer threshold remains the floor-based interval test after substitution.
\end{proof}

\subsection{Conservative Full-Byzantine Reduction}

\begin{proposition}[Conservative full-Byzantine reduction]
\label{prop:app-full-byzantine-reduction}
Assume the equal-weight base model, the semantic-validation probability space
from the preliminaries, a static Byzantine set, and the semantic event
$\SemanticSafetyEvent$. On this event, at most
$\ByzantineBound$ Byzantine identities and at most $\EpistemicBudget$
protocol-compliant false endorsers can support a fixed invalid candidate.
If these protocol-compliant false endorsers are intentionally charged as
fully Byzantine identities for all quorum-intersection purposes, define the
integer effective bound
\[
  F_\delta=\ByzantineBound+
  \left\lfloor\EpistemicBudget\right\rfloor .
\]
Then the classical condition
\[
  \NumValidators\ge 3F_\delta+1
\]
is sufficient to reuse a standard authenticated $3F+1$ BFT quorum argument on
$\SemanticSafetyEvent$. The stronger real-valued condition
\[
  \NumValidators\ge 3(\ByzantineBound+\EpistemicBudget)+1
\]
is also sufficient when interpreted as an integer design constraint. If
$\Prob[\SemanticSafetyEvent]\ge 1-\delta$, the reduction applies with
probability at least $1-\delta$. It is deliberately sufficient and pessimistic, not necessary and
not claimed tight for the base model.
\end{proposition}

\begin{proof}
On $\SemanticSafetyEvent$, the realized number of protocol-compliant false
endorsers for a fixed invalid candidate is an integer no larger than
$\lfloor\EpistemicBudget\rfloor$. Adding the worst-case Byzantine identities
gives the integer effective bound
$F_\delta=\ByzantineBound+\lfloor\EpistemicBudget\rfloor$. If the design
satisfies $\NumValidators\ge 3F_\delta+1$, the standard authenticated BFT
quorum argument for $F_\delta$ fully Byzantine identities applies to the
reduced model. Since
$F_\delta\le \ByzantineBound+\EpistemicBudget$, the real-valued inequality
$\NumValidators\ge 3(\ByzantineBound+\EpistemicBudget)+1$ is a stronger
sufficient condition.

The reduction is pessimistic because it treats non-equivocating semantic
errors as if they were Byzantine equivocations for agreement. It also ignores
$\UnusableBudget$, so it does not provide the liveness condition supplied by
Corollary~\ref{cor:app-threshold-feasibility}.
\end{proof}

\section{Calibration and Evaluation Methodology}
\label{sec:appendix-experimental-details}

This appendix gives the calibration record needed to estimate the
confidence-indexed budgets used by the certificate theorems. It complements
the calibration methodology and illustrative budget calculation in
Section~\ref{sec:evaluation}.

\subsection{Task Sampling}

Calibration begins by fixing a workload domain $D$, a committee $S$, a
certificate threshold candidate, model and tool versions, prompts, retrieval
sources, policy versions, and an expiration date. The task sample is drawn
before validator outputs are inspected. The stratification records executable
infrastructure-policy checks, configuration invariants, code or plan
verification tasks, and controlled synthetic state-transition tasks. Random
splits inside one task family are recorded separately from held-out
workload-family splits.

The sampling record states whether the intended guarantee is pointwise,
domain-conditional, average over a named task distribution, or uniform over a
task class. If tasks are selected adversarially after calibration, an
average-case workload estimate is not reused unless the stated bound is valid
for that selection rule.

\subsection{Semantic Labels}

Each task has a defensible semantic-validity label before it is used to
estimate either budget. Label sources include executable policy
checks, deterministic configuration validators, model checkers, proof
artifacts with stated assumptions, trusted verifiers, or adjudication records
with reviewer provenance. The label record binds the state/context,
candidate transition, policy, evidence package, and admissible action space.

Label uncertainty is separate from validator error. If labels are probabilistic
or adjudicated, the calibration artifact records the adjudication method,
disagreement resolution, excluded cases, and any label-confidence interval.
Tasks whose validity cannot be established within the target domain are
kept out of the budget estimate or reported as a separate uncertainty source.

\subsection{Empirical Quantiles}

For each invalid labeled task, compute the protocol-compliant false-endorsement
support
\[
\begin{aligned}
  F_m
  &=
  \sum_{i\in \HonestSet_S} w_i\,\mathbf{1}^{\mathrm{FE}}_{i,m},\\
  \mathbf{1}^{\mathrm{FE}}_{i,m}
  &=
  \mathbf{1}[J_i(x_m,s_m)=\mathsf{endorse}],
  \quad \SemanticValidity(x_m,s_m)=0 .
\end{aligned}
\]
In the equal-weight base model, $F_m$ is a count. In a weighted calibration,
it is total validator weight. Candidate estimates of $\EpistemicBudget$ use an
upper empirical quantile of the observed $F_m$ values, with the chosen
confidence procedure stated before results are computed.

For each valid labeled task, compute unusable support
\[
  U_m=\sum_{i\in \HonestSet_S} w_i U_i(x_m,s_m),
  \qquad \SemanticValidity(x_m,s_m)=1,
\]
where unusable support includes false rejection, abstention, timeout,
divergent semantic output, and canonicalization failure. Candidate estimates
of $\UnusableBudget$ use an upper empirical quantile of the observed $U_m$
values. The analysis reports the empirical quantile, the upper uncertainty
endpoint, and the budget value actually used for threshold selection.

\subsection{Uncertainty Intervals and Rare-Event Limitations}

The reported estimates $\EstimatedEpistemicBudget$ and
$\EstimatedUnusableBudget$ include uncertainty intervals. The upper
confidence endpoints $\EstimatedEpistemicUpperBudget$ and
$\EstimatedUnusableUpperBudget$ are the conservative values used for threshold
selection. Suitable methods include predeclared bootstrap
intervals, exact binomial or beta-binomial tail bounds for exceedance
probabilities, distribution-free order-statistic bounds, or conservative
concentration inequalities. The method, confidence level, sample size, and
stratification rule are recorded with the estimate.
These intervals account for sampling uncertainty under the stated procedure;
hidden label-system bias, measurement error, or unmodeled common-mode
structure requires a separate assumption or sensitivity analysis.

Rare-event claims require rare-event evidence. A calibration set with only a
small number of invalid or valid examples cannot justify an extremely small
tail probability. If the sample size cannot support the target confidence
level, the artifact reports the weakest defensible upper bound rather than a
desired risk number.

\subsection{Distribution-Shift Checks}

Budget estimates are stress-tested against workload shift before they are used
in a signed risk profile. Required checks include held-out workload
families, shared poisoned evidence, ambiguous policies, common prompt
injection, shared retrieval corruption, model-family-specific traps, and
changes in state or policy representation. Each shifted or adversarial task
identifies the calibration assumption it tests.

A calibration study compares the estimated budgets on calibration tasks,
held-out families, and adversarial suites. A material increase in largest
false-endorsement coalition, unusable support, unsafe-certificate rate, or
no-certificate rate triggers domain narrowing, threshold recomputation,
committee reselection, or recalibration.

\subsection{Committee Provenance}

For every validator configuration, record provider, model identifier,
model-version timestamp, decoding parameters, prompt family, prompt hash,
retrieval source, retrieval-corpus digest, toolchain, tool versions, isolation
policy, jurisdiction, cost model, and capability assumptions. Homogeneous
committees identify the replicated configuration. Heterogeneous committees
state which dimensions vary and which shared dependencies remain.

Provenance is part of the estimate. A change in model, prompt, retrieval
corpus, toolchain, policy version, provider behavior, or evidence source can
change both $\EpistemicBudget$ and $\UnusableBudget$. The signed risk profile
therefore includes an expiration date and rerun criteria tied to these
provenance fields.

\subsection{Reproducibility Data Schema}

Each calibration run produces a machine-readable record with the fields below.
The schema records inputs and estimates, not experimental result claims.

\begin{itemize}
\item \textbf{Run metadata:} run identifier, calibration date, operator,
random seed, code version, analysis-script hash, and risk-profile expiration.
\item \textbf{Domain metadata:} workload domain $D$, task-family label,
sampling rule, split identifier, and guarantee scope.
\item \textbf{Task record:} state/context digest, candidate-transition digest,
policy digest, evidence-package identity, admissible action space, semantic
label, label source, and label-confidence metadata.
\item \textbf{Validator record:} validator identity, committee membership,
weight, provider, model version, prompt hash, retrieval digest, tool versions,
jurisdiction, and isolation policy.
\item \textbf{Judgment record:} signed judgment action, semantic class,
evidence digest, provenance record, latency, timeout status, canonicalization
status, and raw output digest.
\item \textbf{Derived variables:} false endorsement, false rejection,
abstention, timeout, divergent semantic output, canonicalization failure,
largest false-endorsement coalition, $F_m$, and $U_m$.
\item \textbf{Budget estimate:} empirical quantile rule, uncertainty method,
confidence level, sample size, $\widehat{e}_\delta$,
$\widehat{u}_\epsilon$, $\widehat{e}^{+}_\delta$,
$\widehat{u}^{+}_\epsilon$, selected threshold $q$, and theorem inequalities
checked.
\item \textbf{Shift and adversarial checks:} held-out-family identifier,
adversarial condition, changed dependency, observed failure mode, and
recalibration decision.
\end{itemize}

\end{appendices}

\bibliographystyle{unsrt}
\bibliography{references}

@article{pease1980reaching,
  author = {Pease, Marshall and Shostak, Robert and Lamport, Leslie},
  title = {Reaching Agreement in the Presence of Faults},
  journal = {Journal of the ACM},
  year = {1980},
  volume = {27},
  number = {2},
  pages = {228--234},
  doi = {10.1145/322186.322188}
}

@article{lamport1982byzantine,
  author = {Lamport, Leslie and Shostak, Robert and Pease, Marshall},
  title = {The Byzantine Generals Problem},
  journal = {ACM Transactions on Programming Languages and Systems},
  year = {1982},
  volume = {4},
  number = {3},
  pages = {382--401},
  doi = {10.1145/357172.357176}
}

@article{fischer1985impossibility,
  author = {Fischer, Michael J. and Lynch, Nancy A. and Paterson, Michael S.},
  title = {Impossibility of Distributed Consensus with One Faulty Process},
  journal = {Journal of the ACM},
  year = {1985},
  volume = {32},
  number = {2},
  pages = {374--382},
  doi = {10.1145/3149.214121}
}

@article{dwork1988consensus,
  author = {Dwork, Cynthia and Lynch, Nancy and Stockmeyer, Larry},
  title = {Consensus in the Presence of Partial Synchrony},
  journal = {Journal of the ACM},
  year = {1988},
  volume = {35},
  number = {2},
  pages = {288--323},
  doi = {10.1145/42282.42283}
}

@article{schneider1990state,
  author = {Schneider, Fred B.},
  title = {Implementing Fault-Tolerant Services Using the State Machine
    Approach: A Tutorial},
  journal = {ACM Computing Surveys},
  year = {1990},
  volume = {22},
  number = {4},
  pages = {299--319},
  doi = {10.1145/98163.98167}
}

@inproceedings{castro1999practical,
  author = {Castro, Miguel and Liskov, Barbara},
  title = {Practical Byzantine Fault Tolerance},
  booktitle = {Proceedings of the Third Symposium on Operating Systems
    Design and Implementation},
  year = {1999},
  pages = {173--186},
  publisher = {USENIX Association},
  address = {Berkeley, CA, USA}
}

@article{malkhi1998byzantine,
  author = {Malkhi, Dahlia and Reiter, Michael},
  title = {Byzantine Quorum Systems},
  journal = {Distributed Computing},
  year = {1998},
  volume = {11},
  number = {4},
  pages = {203--213},
  doi = {10.1007/s004460050050}
}

@inproceedings{benor1983another,
  author = {Ben-Or, Michael},
  title = {Another Advantage of Free Choice: Completely Asynchronous
    Agreement Protocols},
  booktitle = {Proceedings of the Second Annual ACM Symposium on
    Principles of Distributed Computing},
  year = {1983},
  pages = {27--30},
  doi = {10.1145/800221.806707}
}

@article{civit2023validity,
  author = {Civit, Pierre and Gilbert, Seth and Guerraoui, Rachid and
    Komatovic, Jovan and Vidigueira, Manuel},
  title = {On the Validity of Consensus},
  journal = {arXiv preprint arXiv:2301.04920},
  year = {2023},
  doi = {10.48550/arXiv.2301.04920}
}

@article{ranchal2022basilic,
  author = {Ranchal-Pedrosa, Alejandro and Gramoli, Vincent},
  title = {{Basilic}: Resilient Optimal Consensus Protocols With Benign and
    Deceitful Faults},
  journal = {arXiv preprint arXiv:2204.08670},
  year = {2022},
  doi = {10.48550/arXiv.2204.08670}
}

@article{avelas2024probft,
  author = {Avelas, Diogo and Heydari, Hasan and Alchieri, Eduardo and
    Distler, Tobias and Bessani, Alysson},
  title = {Probabilistic Byzantine Fault Tolerance},
  journal = {arXiv preprint arXiv:2405.04606},
  year = {2024},
  doi = {10.48550/arXiv.2405.04606}
}

@article{avizienis1985nversion,
  author = {Avizienis, Algirdas},
  title = {The N-Version Approach to Fault-Tolerant Software},
  journal = {IEEE Transactions on Software Engineering},
  year = {1985},
  volume = {SE-11},
  number = {12},
  pages = {1491--1501},
  doi = {10.1109/TSE.1985.232171}
}

@article{knight1986experimental,
  author = {Knight, John C. and Leveson, Nancy G.},
  title = {An Experimental Evaluation of the Assumption of Independence in
    Multiversion Programming},
  journal = {IEEE Transactions on Software Engineering},
  year = {1986},
  volume = {SE-12},
  number = {1},
  pages = {96--109},
  doi = {10.1109/TSE.1986.6312924}
}

@article{denisovblanch2026consensus,
  author = {Denisov-Blanch, Yegor and Kazdan, Joshua and Chudnovsky, Jessica
    and Schaeffer, Rylan and Guan, Sheng and Adeshina, Soji and Koyejo, Sanmi},
  title = {Consensus is Not Verification: Why Crowd Wisdom Strategies Fail for
    {LLM} Truthfulness},
  journal = {arXiv preprint arXiv:2603.06612},
  year = {2026},
  doi = {10.48550/arXiv.2603.06612}
}

@article{turkmen2026ensemble,
  author = {Turkmen, Yigit and Buyukates, Baturalp and Bastopcu, Melih},
  title = {Don't Always Pick the Highest-Performing Model: An Information
    Theoretic View of {LLM} Ensemble Selection},
  journal = {arXiv preprint arXiv:2602.08003},
  year = {2026},
  doi = {10.48550/arXiv.2602.08003}
}

@article{zhang2025codetriangle,
  author = {Zhang, Taolin and Ma, Zihan and Cao, Maosong and Liu, Junnan and
    Zhang, Songyang and Chen, Kai},
  title = {Coding Triangle: How Does Large Language Model Understand Code?},
  journal = {arXiv preprint arXiv:2507.06138},
  year = {2025},
  doi = {10.48550/arXiv.2507.06138}
}

@article{he2026openkedge,
  author = {He, Jun and Yu, Deying},
  title = {{OpenKedge}: Governing Agentic Mutation with Execution-Bound Safety
    and Evidence Chains},
  journal = {arXiv preprint arXiv:2604.08601},
  year = {2026},
  doi = {10.48550/arXiv.2604.08601}
}

@article{he2026sovereignLoops,
  author = {He, Jun and Yu, Deying},
  title = {{Sovereign Agentic Loops}: Decoupling {AI} Reasoning from Execution in
    Real-World Systems},
  journal = {arXiv preprint arXiv:2604.22136},
  year = {2026},
  doi = {10.48550/arXiv.2604.22136}
}

@article{he2026pdd,
  author = {He, Jun and Yu, Deying},
  title = {{Protocol-Driven Development}: Governing Generated Software Through
    Invariants and Continuous Evidence},
  journal = {arXiv preprint arXiv:2605.12981},
  year = {2026},
  doi = {10.48550/arXiv.2605.12981}
}

@article{he2026vai,
  author = {He, Jun and Yu, Deying},
  title = {{Verifiable Agentic Infrastructure}: Proof-Derived Authorization for
    Sovereign {AI} Systems},
  journal = {arXiv preprint arXiv:2605.15228},
  year = {2026},
  doi = {10.48550/arXiv.2605.15228}
}

@article{he2026sqa,
  author = {He, Jun and Yu, Deying},
  title = {{Semantic Quorum Assurance}: Collective Certification for
    Non-Deterministic {AI} Infrastructure},
  journal = {arXiv preprint arXiv:2606.08021},
  year = {2026},
  doi = {10.48550/arXiv.2606.08021}
}

@article{he2026sab,
  author = {He, Jun and Yu, Deying},
  title = {{Sovereign Assurance Boundary}: Certificate-Bound Admission for
    Agentic Infrastructure},
  journal = {arXiv preprint arXiv:2606.11632},
  year = {2026},
  doi = {10.48550/arXiv.2606.11632}
}

@misc{pdds2026manifesto,
  author = {He, Jun and Yu, Deying},
  title = {{The Post-Deterministic Manifesto}: A New Foundation for Trustworthy Autonomous Infrastructure},
  year = {2026},
  howpublished = {arXiv preprint arXiv:2606.01722}
}

\end{document}